%% file: main.tex
\documentclass{article}

\PassOptionsToPackage{numbers, compress}{natbib}

\usepackage[preprint]{neurips_2026}

\usepackage[utf8]{inputenc} 
\usepackage[T1]{fontenc}    
\usepackage{hyperref}       
\usepackage{url}            
\usepackage{booktabs}       
\usepackage{amsfonts}       
\usepackage{nicefrac}       
\usepackage[nopatch=footnote]{microtype}      
\usepackage{xcolor}         
\usepackage{color_edits}
\usepackage{float}
\usepackage{upgreek}
\usepackage{amsbsy}
\usepackage{amssymb}
\usepackage{cancel}
\usepackage{amsmath}
\usepackage{amsthm}
\usepackage{graphicx}
\usepackage{setspace}
\usepackage{mathtools}
\usepackage{titlesec}
\usepackage{enumitem}
\usepackage{caption}
\usepackage{thmtools}
\usepackage{thm-restate}
\usepackage{subcaption}
\usepackage{adjustbox,multirow}
\usepackage[noabbrev]{cleveref}
\usepackage[ruled,vlined]{algorithm2e}
\usepackage{wrapfig}
\usepackage{titletoc} 
\input{math_macros}

\hypersetup{
    colorlinks,
    linkcolor={red!100!black},
    citecolor={blue!100!black},
    urlcolor={blue!80!black}
}
\usepackage[most]{tcolorbox} 

\definecolor{pirategold}{HTML}{BF9B30}
\definecolor{orangesignal}{HTML}{FE6100}
\definecolor{highlightsignal}{HTML}{FFB000}
\definecolor{purplesignal}{HTML}{785ef0}
\definecolor{bluesignal}{HTML}{648fff}
\definecolor{pinksignal}{HTML}{dc267f}

\definecolor{citationcolor}{HTML}{648fff} 
\definecolor{linkcolorcustom}{HTML}{648fff}

\hypersetup{ 
    colorlinks=true,       
    linkcolor=pinksignal,
    citecolor=citationcolor, 
    urlcolor=linkcolorcustom 
}

\title{Learning Orthonormal Bases for Function Spaces}

\author{%
  Hamidreza Kamkari\\
  MIT CSAIL\\
  Cambridge, MA\\
  \texttt{kamkarih@mit.edu} \\
  \And
  Mohammad Sina Nabizadeh\\
  MIT CSAIL\\
  Cambridge, MA \\
  \texttt{sinabiz@mit.edu} \\
  \AND
  Justin Solomon \\
  MIT CSAIL\\
  Cambridge, MA\\
  \texttt{jsolomon@mit.edu} \\
}

\begin{document}

\maketitle

\begin{abstract}
Infinite-dimensional orthonormal basis expansions play a central role in representing and computing with function spaces due to their favorable linear algebraic properties. However, common bases such as Fourier or wavelets are fixed and do not adapt to the structure of a given problem or dataset. In this paper, we aim to represent these bases with neural networks and optimize them. Our key idea is that any target infinite-dimensional orthonormal basis can be viewed either as a point on the Lie manifold of the orthogonal group, or equivalently, as the endpoint of a continuous path on that manifold that connects a reference basis, e.g.\ Fourier, to that target. 
Paths on the Lie manifold satisfy ordinary differential equations (ODEs) governed by skew-adjoint integral operators.
Using neural networks to define finite-rank generators of such ODEs allows us to parameterize and optimize orthonormal bases in function space.
While relying on finite-rank generators to model infinite operators might seem restrictive, we prove a universality result: even with a rank-2 generator, the integrated solutions of the ODE are dense in the orthogonal group under the appropriate operator topology. In other words, for any target orthonormal basis, there exists a path originating from a reference basis and driven by finite-rank generators that gets arbitrarily close to that target basis. We demonstrate the flexibility of our framework by transforming the Fourier basis into the principal components of a functional dataset, eigenfunctions of linear operators, or dynamic modes of energy-preserving physical simulations.
\end{abstract}

\input{neurips_sections/intro}

\input{neurips_sections/related_work}
\input{neurips_sections/math_setup}
\input{neurips_sections/theory1}

\input{neurips_sections/theory2}
\input{neurips_sections/implementation}
\input{neurips_sections/experiments}
\input{neurips_sections/conclusions}

\begin{ack}
The MIT Geometric Data Processing Group acknowledges the generous support of National Science Foundation grants IIS2335492 and OAC2403239, from the CSAIL Future of Data and FinTechAI programs, from the MIT--IBM Watson AI Laboratory, from the Wistron Corporation, from the MIT Generative AI Impact Consortium, from the Toyota--CSAIL Joint Research Center, and from Schmidt Sciences. The authors acknowledge further support from SideFX Software. The authors additionally declare that they have no competing interests.

We thank Cl\'ement Jambon, Artem Lukoianov, Christopher Scarvelis, Ana Dodik, Alice Petrov, Dolores Garcia Marti, and Ahmed Mahmoud for providing valuable feedback on our work. Hamidreza Kamkari extends a special thank you to  Yeganeh Gharedaghi for her instrumental contributions during the early stages of this project and her invaluable help with ideation.
\end{ack}

\bibliographystyle{plainnat}
\bibliography{main}

\newpage
\appendix

\input{appendix/existence_proof}
\newpage
\input{appendix/diagonalization-theory}

\newpage
\input{appendix/lie-integration}

\newpage
\input{appendix/experimental-details}



\end{document}

%% file: math_macros.tex
\newcommand{\reptitleinner}{} 

\newtheorem*{reptheoreminner}{\reptitleinner}
\newcommand{\newreptheorem}[2]{%
\newenvironment{rep#1}[1]{%
 \def\reptitleinner{#2 \ref*{##1}}%
 \begin{reptheoreminner}}%
 {\end{reptheoreminner}}}

\newtheorem{theorem}{Theorem}

\newreptheorem{theorem}{Theorem}

\newreptheorem{corollary}{Corollary}
\newtheorem{lemma}[theorem]{Lemma}

\newreptheorem{proposition}{Proposition}

\theoremstyle{definition}
\newtheorem{definition}{Definition}
\newreptheorem{definition}{Definition}

\crefname{equation}{}{}
\crefname{proposition}{Proposition}{Propositions}
\crefname{section}{Section}{Sections}
\crefname{appendix}{Appendix}{Appendices}
\crefname{lemma}{Lemma}{Lemmas}
\crefname{assumption}{Assumption}{Assumptions}
\crefname{algorithm}{Algorithm}{Algorithms}
\crefname{theorem}{Theorem}{Theorems}
\crefname{corollary}{Corollary}{Corollaries}
\crefname{insight}{Insight}{Insights}
\crefname{definition}{Definition}{Definitions}
\crefname{figure}{Figure}{Figures}
\crefname{table}{Table}{Tables}

\DeclareMathOperator{\E}{\mathbb{E}}

\DeclareMathOperator{\R}{\mathbb{R}}  

\def\ddefloop#1{\ifx\ddefloop#1\else\ddef{#1}\expandafter\ddefloop\fi}
\def\ddef#1{\expandafter\def\csname bb#1\endcsname{\ensuremath{\mathbb{#1}}}}
\ddefloop ABCDEFGHIJKLMNOPQRSTUVWXYZ\ddefloop
\def\ddefloop#1{\ifx\ddefloop#1\else\ddef{#1}\expandafter\ddefloop\fi}
\def\ddef#1{\expandafter\def\csname b#1\endcsname{\ensuremath{\mathbf{#1}}}}
\ddefloop ABCDEFGHIJKLMNOPQRSTUVWXYZ\ddefloop
\def\ddef#1{\expandafter\def\csname c#1\endcsname{\ensuremath{\mathcal{#1}}}}
\ddefloop ABCDEFGHIJKLMNOPQRSTUVWXYZ\ddefloop
\def\ddef#1{\expandafter\def\csname h#1\endcsname{\ensuremath{\hat{#1}}}}
\ddefloop ABCDEFGHIJKLMNOPQRSTUVWXYZ\ddefloop
\def\ddef#1{\expandafter\def\csname wh#1\endcsname{\ensuremath{\widehat{#1}}}}
\ddefloop abcdefghijklmnopqrstuvwxyz\ddefloop
\def\ddef#1{\expandafter\def\csname hc#1\endcsname{\ensuremath{\widehat{\mathcal{#1}}}}}
\ddefloop ABCDEFGHIJKLMNOPQRSTUVWXYZ\ddefloop
\def\ddef#1{\expandafter\def\csname t#1\endcsname{\ensuremath{\widetilde{#1}}}}
\ddefloop ABCDEFGHIJKLMNOPQRSTUVWXYZ\ddefloop
\def\ddef#1{\expandafter\def\csname tc#1\endcsname{\ensuremath{\widetilde{\mathcal{#1}}}}}
\ddefloop ABCDEFGHIJKLMNOPQRSTUVWXYZ\ddefloop

\usepackage{prettyref}
\usepackage{xcolor}

\renewcommand{\hat}{\widehat}

%% file: neurips_sections/intro.tex
\section{Introduction}
\label{sec:intro}

Orthonormal basis (ONB) expansions
are basic building blocks of various fields. They help us represent functions, run numerical simulations, and analyze data in machine learning. 
The choice of basis governs not only 
approximation 
accuracy but also the computational cost of downstream tasks. Often, the most useful basis is adapted to the problem at hand. For instance, in data analysis, principal components optimally capture variance, allowing one to analyze the data by looking at only the first few basis elements \cite{stoica2005spectral, abdi2010principal}. 
In physics and geometry, the first few spectral eigenfunctions approximate the dynamic or geometric modes of the system \cite{courant1954methods, schmid2010dynamic}. In signal processing, Fourier and wavelet series \cite{mallat1999wavelet, bremaud2002mathematical} help pinpoint various signal characteristics.

Neural networks flexibly parameterize complex function spaces and adapt to specific tasks via gradient-based optimization. Here, we leverage their expressiveness for the space of ONBs and ask:
\begin{center}
 \emph{Can we parameterize infinite series of orthonormal and continuous bases with neural networks?} 
\end{center}

Classical ONBs such as Fourier series provide a good starting point: they are infinite and orthonormal by construction, inner products decompose cleanly into coefficient sums, and projection  
requires a single inner product.
They are also continuous and discretization-free: they represent functions via linear combination of other functions on the entire domain, rather than as discrete values on a grid.
Thus, the resulting space can be evaluated, differentiated, or integrated at any point in the domain.

Any two ONBs, up to sign, are connected by a special orthogonal change-of-basis operator. Therefore, rather than directly modeling the space of ONBs, we can parameterize the equivalent orthogonal map that inputs a reference ONB---e.g.\ the Fourier series---and outputs a basis tailored to the application at hand. This construction allows the target ONB to inherit desirable properties of this reference: remaining infinite, continuous, and orthonormal.
\autoref{fig:teaser} (a) shows how parameterizing this operator allows us to map the Fourier basis to the principal components of the celebrity face dataset~\cite{kist_andreas_m_2021_5561092}, recover the spectrum of a linear operator, or extract the dynamic modes of a system. 

To tractably parameterize this change-of-basis operator, we draw on Lie theory \cite{hall2013lie}. Rather than optimizing an orthogonal operator directly, we view it as a continuous path across the manifold of the orthogonal group (see \autoref{fig:teaser}; c). Geometrically, this path corresponds to a sequence of infinitesimal rotations in function space, modeled as the flow of an ODE governed by a time-dependent, skew-adjoint infinitesimal generator. 
The orthogonal group forms a highly constrained manifold that is hard to parameterize, but the space of skew-adjoint generators is linear and easier to model with a neural network.
The remaining challenge is that even the space of generators is infinite-dimensional. 
To resolve this, we prove in \autoref{sec:so-inf} that any path on this infinite group can be arbitrarily well-approximated using a \emph{finite-rank} generator. 
Ultimately, this reduces the task of parameterizing and optimizing an infinite ONB to the following: we use a neural network to model a finite-rank, skew-adjoint generator and integrate the resulting ODE to smoothly change a reference basis into a target basis. By optimizing the network's parameters via neural ODE backpropagation \cite{chen2018neural}, we model a continuous path on the orthogonal group that adapts to our needs.

\begin{figure}\captionsetup{font=footnotesize}
    \centering
    \includegraphics[width=\textwidth]{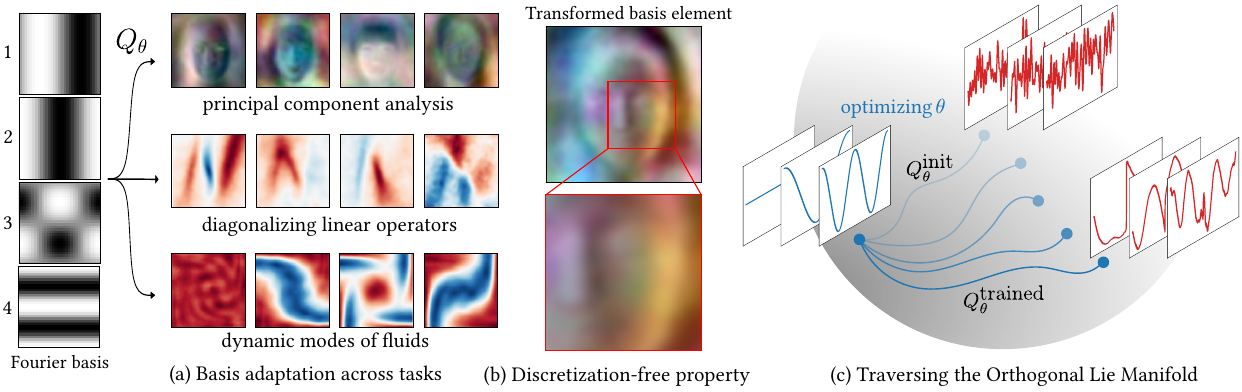}
    \caption{\textbf{Method and Applications Overview.} \textbf{(a)} Parameterizing a change-of-basis map $Q_\theta$ allows us to adapt a reference ONB to various applications: eigenfaces that describe the variance of the data, visualizing the training dynamics of a non-linear neural function by diagonalizing its neural tangent kernel, or dynamic modes of fluid-flow. \textbf{(b)} The learned bases inherit properties of the reference such as being discretization-free. \textbf{(c)} We model paths on the Lie manifold of the infinite orthogonal group and gradually transform an initial basis by traversing this manifold; by optimizing the path, we are able to learn ONBs.}
    \vspace{-0.7cm}
    \label{fig:teaser}
\end{figure}

Our contributions are: 
{(i)} reformulating orthonormal bases as a rotation path on the infinite-dimensional orthogonal group, achieved by solving an ODE; 
{(ii)} a universality theorem showing that a large family of rotations (specifically, those generated by Hilbert-Schmidt operators) can be approximated by the flow of an ODE governed by a low-rank skew-adjoint operator; 
{(iii)} a practical neural parameterization of this generator, trained end-to-end via ODE-based backpropagation, yielding an orthonormal, continuous, and infinite basis that can be adapted to variational problems; 
{(iv)} a principled variational objective for diagonalizing linear operators, demonstrated in two applications by recovering the eigenfunctions of the data covariance operator and the neural tangent kernel; and 
{(v)} applications beyond eigenproblems to modeling dynamical systems, demonstrating that our basis transformations can approximate the Koopman operator of fluids.

%% file: neurips_sections/related_work.tex
\section{Related Work} \label{sec:related_work}

\textbf{Basis expansions.} 
A principled way 
to discover a problem-adapted basis is to compute the eigenfunctions of an underlying linear operator. In physics, eigenfunctions of differential operators---such as spherical harmonics, Bessel functions, and Hermite polynomials---diagonalize the Laplacian 
and form the foundation of spectral methods for solving differential equations~\cite{gottlieb1977numerical, trefethen2000spectral, levy2006laplace,chang2024shape}. Alternatively, operators can be data-driven rather than analytic; for example, in principal component analysis (PCA), the principal components of a stochastic process are the eigenfunctions of its covariance operator~\cite{gerbrands1981relationships, abdi2010principal}. In other cases, bases are engineered rather than linked to operators; wavelet bases are one such example which are designed to capture local signal characteristics \cite{mallat1999wavelet}. 
Despite their different origins, ONBs are useful and unanimous across various fields involving numerical processing. 

\textbf{Eigendecomposition.} 
Solving an eigenproblem is a standard method for obtaining custom bases. The classical approach involves discretizing the operator as a matrix on a grid or mesh and applying numerical eigensolvers~\citep{saad2011numerical, golub2013matrix, stoica2005spectral}. However, these approaches inevitably incur discretization error in continuous state spaces. While solving the discretized system and applying post-hoc interpolation is a possible workaround~\citep{bengio2003out}, recent work instead suggests using neural networks to learn continuous eigenfunctions directly by solving variational problems. For instance, Spectral Inference Networks~\citep{pfau2018spectral} pair this approach with an explicit orthogonalization step, whereas NeuralEF~\citep{deng2022neuralef} avoids it by extending EigenGame~\citep{gemp2020eigengame} to function space. Further extending this paradigm, Ben-Shaul et al.~\citep{ben2023deep} target differential self-adjoint operators on free-form domains, and Azmoodeh et al.~\citep{azmoodeh2025continuous} address continuous-time PCA via the directly representing the ONBs. While these neural methods are discretization-free, they typically recover only a small, finite number of eigenfunctions without strictly enforcing orthogonality. Our framework, by contrast, maintains exact orthonormality across the full eigenfunction sequence by always remaining on the Lie manifold.

\textbf{Non-linear representations of function spaces.} 
An alternative to using orthonormal bases is to forego the linear structure of function spaces entirely in favor of highly expressive, non-linear neural parameterizations. Neural operators~\cite{li2020neural, li2020fourier, lu2021learning, kovachki2023neural, li2024physics}, implicit neural representations (INRs)~\citep{sitzmann2020implicit, mildenhall2021nerf, xie2022neural, dupont2022data}, and physics-informed neural networks (PINNs)~\cite{raissi2019physics} loosely fall into this category and have achieved success across image processing, computer graphics, and physical simulation~\citep{modi2024simplicits, sharp2023data}. However, while these representations adapt remarkably well to specific problems, they sacrifice the linear algebraic structure of the function space: linear operations in function space---such as addition or applying a differential operator---map to highly non-linear and unpredictable transformations in the network's weight space~\citep{xu2022signal, yuce2022structured, de2023deep}. While recent architectures like the Fourier tensor network~\citep{ashtari2026futon} have begun to address this by, for example, explicitly structuring INRs with orthogonal Fourier bases and low-rank tensor coefficients, our framework offers an alternative: we leverage the adaptive power of neural networks while strictly preserving the linear algebraic structure of the function space by parameterizing bases. We even demonstrate how to bridge these two worlds; in one of our applications, we take an arbitrary non-linear neural network and project it back into a structured function space by learning the continuous ONB that diagonalizes its neural tangent kernel.

%% file: neurips_sections/math_setup.tex
\section{Mathematical Preliminaries}\label{sec:math_setup}

\textbf{Function spaces.} We work with functions defined on a $d$-dimensional compact domain $\Omega$ equipped with finite measure $\mu$. We consider the real-valued Hilbert space $H \coloneqq L^2(\Omega, \mu)$, with its inner product and norm defined in standard fashion:
    $\langle f,g \rangle \coloneqq \int_\Omega f(\omega) g(\omega) \, \mu(d\omega) = \mathbb{E}_{\omega \sim \mu} [f(\omega) g(\omega)]$ and $\|f\|^2 \coloneq \langle f,f \rangle.$
Since $\mu$ is finite, we may without loss of generality assume it is a probability measure, allowing for Monte Carlo estimation of the inner product.
We assume $H$ is \emph{separable}, meaning it admits a countable complete ONB $\{ \varphi_n \}_{n=1}^\infty$. Concretely, for any $f \in H$,  
$\lim_{n \to \infty} \| f - \sum_{m=1}^n \langle f, \varphi_m\rangle \cdot \varphi_m \| = 0.$ 
For simplicity, we assume $\Omega$ possesses a standard Euclidean metric; thus, a valid ONB $\{ \varphi_n \}_{n=1}^\infty$ (among many) can be the standard $d$-dimensional real Fourier basis.
This basis may be suboptimal for a given problem, motivating us to change it to a problem-adapted one. Henceforth, we will use $\{ \varphi_i \}_{i=1}^\infty$ to denote the initial (reference) basis.

\textbf{Linear operators.}
A bounded linear transformation on $H$ is a \emph{linear operator}. Any such operator $A$ admits a unique adjoint denoted by $A^\ast$, defined 
such that:  $\langle f, Ag\rangle = \langle A^\ast f, g\rangle, \forall f, g \in H$. When $A$ can be written in an integral kernel form $A[\cdot, \cdot]: \Omega \times \Omega \to \R$, its action follows the identity:
\begin{equation}
    (Af)(x) = \int_{\Omega} A[x, y] \, f(y) \, \mu(dy); \quad \text{consequently,} \quad \forall x, y \in \Omega: A^\ast[x, y] = A[y, x].
\end{equation}
$A$ is \emph{skew-adjoint} if $A^\ast = -A$, or equivalently, $\forall x, y \in \Omega: A[x, y] = -A[y, x]$.

The space of operators themselves have norms and topology. Here, we use the \emph{Hilbert-Schmidt} norm:
\begin{equation}
    \| A \|_{\text{HS}}^2 \coloneqq \sum_{n=1}^\infty \|A\varphi_n\|^2 = \sum_{n=1}^\infty \langle A\varphi_n, A\varphi_n\rangle = \sum_{n=1}^\infty \langle \varphi_n, A^\ast A\varphi_n\rangle  \eqqcolon \operatorname{tr}(A^\ast A).
\end{equation}
where $\{ \varphi_n \}_{n=1}^\infty$ can be any complete ONB. An operator $A$ is {Hilbert-Schmidt} if this norm is finite---or equivalently, if $A^\ast A$ is \emph{trace-class}, meaning its trace is finite and well-defined. 

\textbf{Infinite-dimensional orthogonal group.} Any two ONBs, $\{\varphi_i\}_{i=1}^\infty$ and $\{\phi_i\}_{i=1}^\infty$, are related by an orthogonal change-of-basis operator $Q: H \to H$, where $Q^\ast Q = I$.

To make our search for a basis tractable, we borrow a concepts from Lie theory \cite{hall2013lie}: a large subset of orthogonal operators $Q$ can be characterized by a smooth manifold obtained via the exponentiation of a skew-adjoint operator. Concretely, let $Q(t)$ be a family of operators obeying the following ODE:
\begin{equation}\label{eq:exponential-map-basic}
\frac{d}{dt} Q(t) = K Q(t), \quad Q(0) = I, \qquad \text{where } K: H \to H \text{ is skew-adjoint}.
\end{equation}
The solution, $Q(t)$, remains orthogonal because the inner product $Q^\ast Q$ is preserved:
\begin{equation}\label{eq:orthoproof}
\frac{d}{dt} (Q^\ast Q) = \dot{Q}^\ast Q + Q^\ast \dot{Q} = Q^\ast K^\ast Q + Q^\ast K Q = Q^\ast (K^\ast + K)Q = 0.
\end{equation}
Searching over the smooth family of generators $K$ instead of $Q$, establishes a tractable search space. Formally, this allows us to work with the following continuous group:
\begin{definition}
The \textit{Hilbert-Schmidt orthogonal group}, denoted $\mathcal{SO}(HS)$, is the space of operators that can be written as solutions to the ODE in \eqref{eq:exponential-map-basic}, where $K$ is additionally Hilbert-Schmidt.
\end{definition}

%% file: neurips_sections/theory1.tex
\section{Approximation Theory of the Orthogonal Group}\label{sec:so-inf}

Our goal is to represent an orthogonal change-of-basis operator $Q \in \mathcal{SO}(HS)$ using the parameters $\theta$ of a neural network. As a first step, rather than directly modeling $Q$, we model its skew-adjoint infinitesimal generator $K_\theta(t)$ as a time-evolving operator:
\begin{equation} \label{eq:init-value-problem}
    \frac{d}{dt} Q(t) = K_\theta(t) Q(t), \qquad Q(0) = I,
\end{equation}
integrated to time $T > 0$.
Even when $K_\theta(t)$ is time-dependent,~\eqref{eq:orthoproof} shows $Q(t)$ is an orthogonal time-evolving operator. 
In a slight abuse of notation, we will denote this operator by $Q_\theta := Q(T)$, dropping the time evaluation and using $\theta$ emphasize its parameterization.

We have reduced the problem of designing an orthogonal operator to designing a skew-adjoint operator with a neural network. If the infinitesimal generator $K_\theta$ were to have finite rank $r=2$, we could represent its skew-adjoint kernel using two \emph{unrestricted} multilayer perceptrons (MLPs):
\begin{equation} \label{eq:rank-2-def}
    K_\theta(t)[x, y] = \text{MLP}_1(x, t; \theta) \cdot  \text{MLP}_2(y, t; \theta) - \text{MLP}_2(x, t; \theta) \cdot \text{MLP}_1(y, t; \theta).
\end{equation}

Although designing the generator to have a finite rank might appear as a restrictive choice, integrating low-rank dynamics over time allows one to approximate \emph{any} transformation in $\mathcal{SO}(HS)$. This is true even when the rank is as small as $r=2$. We now make this concrete:

\begin{restatable}[Approximating the Orthogonal Group via Rank-2 Generators]{theorem}{ApproximationTheorem}\label{thm:universal-appx}
Let $\mathcal{Q}_T$ denote the set of all operators $Q(T)$ obtained by solving the initial value problem \eqref{eq:init-value-problem} at time $T > 0$, where the generator $K(t)$ is uniformly Hilbert-Schmidt, continuous, and rank-$2$. Then, $\mathcal{Q}_T$ is dense in $\mathcal{SO}(HS)$ with respect to the Hilbert-Schmidt norm topology.
In other words, every rotation in this infinite orthogonal group is arbitrarily well-approximated by some time-evolving rank-2 generator.
\end{restatable}
\begin{proof}[Proof Sketch]
We first approximate an arbitrary $Q \in \mathcal{SO}(HS)$ with an operator that acts non-trivially only on a finite $n$-dimensional subspace. Restricted to this subspace, we have a constructive proof: this operator is isomorphic to an $n \times n$ orthogonal matrix. A classical result from linear algebra states that any finite-dimensional rotation can be factorized into a sequence of planar Givens' rotations \cite{golub2013matrix}. Each planar rotation is generated by a fixed rank-2 skew-adjoint operator and by using a time-dependent scalar control function, we can chain these stationary generators in sequence; effectively ``turning on'' each planar rotation one after the other. This constructs a single continuous time-varying rank-2 generator whose time-ordered integral results in the target finite-dimensional rotation. Finally, taking the limit of $n \to \infty$ in the Hilbert-Schmidt operator topology extends this approximation to the entire infinite-dimensional group $\mathcal{SO}(HS)$; for full details, see \autoref{appx:theorem1-proof}.
\end{proof}

\autoref{thm:universal-appx} thus establishes that an entire change-of-basis can be implicitly parameterized through a finite-rank $K_\theta$ as in~\eqref{eq:rank-2-def}.
Due to the universal approximation theorem of MLPs (the version that uses compactness of $\Omega$ by Hornik et al. \cite{hornik1989multilayer}), 
even the rank $r=2$ model in \eqref{eq:rank-2-def} with a sufficiently overparameterized $\theta$ can implicitly represent an arbitrary basis in function space.

%% file: neurips_sections/theory2.tex
\section{Variational Problems: Adapting the Basis to Various Tasks}\label{sec:variational}

The previous sections establish how a path on the orthogonal group manifold can be parameterized with a neural network and traversed while preserving structure. Our next task is to put this parameterization to use. To that end, we derive variational objective functions that reward or penalize paths on the Lie manifold and allow us to optimize it and find useful bases.

\textbf{Eigenproblems.} One class of variational problems is finding $Q_\theta$ that \emph{diagonalizes} an infinite-dimensional linear operators $A:H\to H$; this is akin to finding eigenvectors of $n\times n$ matrices in the finite case. In our formulation, we can freely search the space of operators $Q \in \mathcal{SO}(HS)$ to find one that diagonalizes $A$. To do so, we reformulate the eigenproblem in the variational language:

\begin{restatable}{theorem}{DiagonalizationTheorem}
\label{thm:diagonalization}
Let $A$ be a trace-class positive semidefinite linear operator, and let $\{ p(i) \}_{i=1}^\infty$ be a strictly decreasing sequence of positive weights such that $\sum_{i=1}^\infty p(i) = 1$, and $\{ \varphi_i \}_{i=1}^\infty$ be any orthonormal basis on $H$. Consider $Q^{\text{opt}}$ as the solution to the following variational problem:
\begin{equation}\label{eq:variational-optim-eig}
\begin{aligned}
     \max_{Q \in \mathcal{SO}(HS)} & \quad \sum_{i=1} p(i) \cdot \langle Q \varphi_i, A Q \varphi_i\rangle
\end{aligned}
\end{equation} 
Then, $\phi_i^\star \coloneqq Q^{\text{opt}} \varphi_i$ are the eigenfunctions of $A$. That is, $A \phi^\star_i = \lambda_i \phi^\star_i$ for all $i \ge 1$, where 
$\lambda_1 \ge \lambda_2 \ge \lambda_3 \ge \cdots$ are the corresponding eigenvalues.
\end{restatable}

See \autoref{appx:diagonalization-theory} for a proof. 
Interpreting $p(\cdot)$ as a probability distribution on indices $i$ 
allows us to write the objective as an expectation. Furthermore, replacing $Q$ with $Q_\theta$ changes an otherwise intractable search over $\mathcal{SO}(HS)$ into one over parameters $\theta$, compatible with stochastic optimization:
\begin{align} \label{eq:unconstrained-optim}
    \max_{\theta} & \quad \mathbb{E}_{i \sim p} \langle Q_\theta \varphi_i, A Q_\theta \varphi_i \rangle.
\end{align}
Substituting various operators for $A$ allows us to solve different eigenproblems in function space. We provide an example below; \autoref{appx:diagonalization-theory} details another involving the neural tangent kernel.

\textbf{PCA in function space.}
We formulate PCA in function space as an eigenproblem.\footnote{PCA in function space is sometimes referred to as the Karhunen-Lo\'eve expansion in the literature \cite{stark1986probability, ghanem2003stochastic, wang2016functional}.} In this setting, we assume our data consists of random functions $X$ drawn from an $\Omega$-indexed stochastic process. This process can be formally viewed as a probability distribution $\mathbb{P}$ over the Hilbert space $H$. For simplicity of exposition, we assume that these functions have zero mean, $\mathbb{E}_{X \sim \mathbb{P}}[X] = 0$; we defer the treatment of the non-centered case to \autoref{appx:diagonalization-theory}.

The covariance operator $\cC : H \to H$ is defined by its action on any function $f \in H$ as $\cC f = \mathbb{E}_{X \sim \mathbb{P}}[X \langle X, f \rangle]$. Plugging this definition into 
\eqref{eq:unconstrained-optim} yields the following derivation:
\begin{equation}
\label{eq:unconstrained-pca}
\begin{aligned}
    \mathbb{E}_{i \sim p} \langle Q_\theta \varphi_i, \cC Q_\theta \varphi_i \rangle = \mathbb{E}_{i \sim p} \big\langle Q_\theta \varphi_i, \mathbb{E}_{X \sim \mathbb{P}} \big[ X \langle X, Q_\theta \varphi_i \rangle \big] \big\rangle = \mathbb{E}_{i \sim p, X \sim \mathbb{P}} \langle X, Q_\theta \varphi_i\rangle^2.
\end{aligned} 
\end{equation}
This objective maximizes the explained variance of the stochastic process $\mathbb P$. Practically, the optimization algorithm is as follows: sample $i \sim p$,  
retrieve the corresponding initial (Fourier) basis function $\varphi_i$, apply $Q_\theta$, sample a random function $X \sim \mathbb{P}$, estimate the squared projection $\langle X, Q_\theta \varphi_i\rangle^2$ using random quadrature points $\hat{\Omega}$, finally maximize via stochastic gradient ascent on $\theta$.

\textbf{Koopman operators.} 
We also cover an application that cannot be formulated as an eigenproblem.
For a dynamical system governed by a state transformation $\Psi: \Omega \to \Omega$, the associated Koopman operator $\mathcal{U}$ \cite{koopman1931hamiltonian} acts on observable functions via composition $\mathcal{U}f \coloneqq f \circ \Psi$. When
$\Psi$ is 
volume-preserving (e.g.\ incompressible flow), 
$\mathcal{U}$ becomes orthogonal. Because 
$Q_\theta$ is orthogonal by design, it provides a suitable inductive bias to learn these dynamics. To that end, we directly fit $Q_\theta$ to the Koopman $\cU$ by minimizing the expected projection error:
\begin{equation} \label{eq:koopman-training}
\min_{\theta} \quad \E_{i\sim p} \| Q_\theta \varphi_i - \varphi_i \circ \Psi\|^2.
\end{equation}
Since both $Q_\theta$ and $\cU$ are linear, training them to match on the complete basis $\{ \varphi_i \}$ is sufficient to fit $Q_\theta$ to $\cU$. In turn, the learned ONB $\phi_i = Q_\theta \varphi_i$ represents the \emph{dynamic modes} of the system.

This formulation becomes particularly valuable when  
$\Psi$ cannot be numerically simulated in a strictly volume-preserving manner. In such cases, learning $Q_\theta$ that is orthogonal by design guarantees energy preservation in the system---$\forall f \in H: \|Q_\theta f\|^2 = \| f \|^2 $; even when $\| f \circ \Psi \|^2 \neq \|f\|^2$ numerically. 
In \autoref{fig:taylor_vortices}, we demonstrate this via a simple fluid dynamical flow: by repeatedly transforming a Fourier basis using $Q_\theta$, we generate future state snapshots that strictly preserve norm, despite being trained on imperfect numerical approximations of $\Psi$.

%% file: neurips_sections/implementation.tex
\section{Algorithmic Details}
\label{sec:implementation_details}

\begin{figure}\captionsetup{font=footnotesize}
    \centering
    \includegraphics[width=\textwidth]{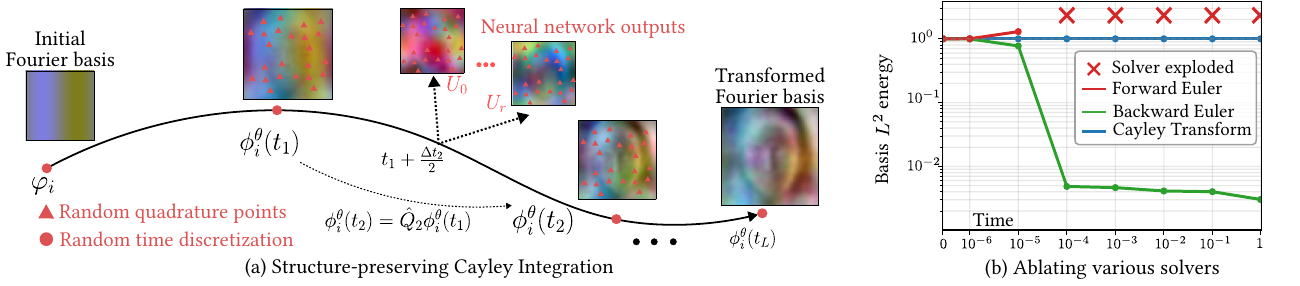}
    \caption{\textbf{Numerical Pipeline.} \textbf{(a)} A Fourier basis is gradually transformed into principal functions of CelebA dataset. We select quadrature points $\hat{\Omega} = \{ \omega_j \}_{j=1}^D \sim \mu$ and time-discretization $\{t_0=0, \ldots, t_L=T\}$ and perform $L$ Cayley steps. Each step takes $\cO(r^3 + r^2D)$ that exploits the rank $r$ structure of the generator. \textbf{(b)} Using off-the-shelf ODE solvers result in loss or explosion of the function norm, by contrast, Cayley integration keeps it intact.}
    \label{fig:low_rank_function_space_rotation}    
    \vspace{-0.5cm}
\end{figure}

Our objectives require the action of $Q_\theta$ on the initial basis functions $\varphi_i$. 
Defining $\phi_i^\theta(t) \coloneqq Q_\theta(t) \varphi_i$ and substituting into \eqref{eq:init-value-problem}, we obtain
\begin{equation}\label{eq:basis-ode}
    \frac{d}{dt} \phi_i^\theta(t) = K_\theta(t) \phi_i^\theta(t), \qquad \phi_i^\theta(0) = \varphi_i.
\end{equation}
In principle, one could integrate \eqref{eq:basis-ode} with off-the-shelf ODE solvers and optimize $\theta$. However, standard solvers like forward and backward Euler accumulate numerical error over time and violate the orthonormal structure; we have included a comparison in  \autoref{fig:low_rank_function_space_rotation} with more details in \autoref{appx:lie-integration}. We therefore use a structure-preserving Lie group integrator \cite{diele1998cayley, iserles2001cayley, celledoni2014introduction} that preserves inner products.

\textbf{Structure-preserving ODE integration.}
Our Lie integrator is the Cayley method \cite{iserles2001cayley}, a structure preserving method which discretizes time as $t_0 = 0 \le t_1 < \ldots < t_L = T$ and advance the system from $t_{\ell-1}$ to $t_\ell$. For each $\ell$, the generator is approximated by a midpoint rule $\hat{K}_\ell = K_\theta(t_{\ell-1} + \tfrac{\Delta t_\ell}{2})$ and the system is evolved by applying $Q_\ell = \exp\left( \Delta t \cdot \hat{K}_\ell\right)$. An important step in the Cayley method is that it replaces the exponential with its $(1,1)$-Pad\'e approximant which preserves inner products:
\begin{equation}\label{eq:cayley-action}
    \phi_i^\theta(t_\ell) \approx \hat{Q}_\ell\, \phi_i^\theta(t_{\ell-1}), \qquad \hat{Q}_\ell := \left(I - \tfrac{\Delta t}{2}\hat{K}_\ell\right)^{-1}\left(I + \tfrac{\Delta t}{2}\hat{K}_\ell\right) \approx Q_\ell.
\end{equation}
One can verify algebraically that if $\hat{K}_\ell$ is skew-adjoint, then $\hat{Q}_\ell$ is orthogonal~\citep{golub2013matrix, gallier2006remarks}, so the Cayley update preserves the norms and inner products of the basis exactly at each step. 

Applying \eqref{eq:cayley-action} requires discretization
at some quadrature points $\hat{\Omega} = \{\omega_j\}_{j=1}^D \sim \mu$, where the $\omega_j$ are drawn via stratified sampling.
To avoid overfitting to any discretization, we resample the points before each ODE solve.
Hence, discretely, $\hat{\phi}_i^\theta(t) \in \R^D$, $\hat{Q}_\ell \in \R^{D \times D}$, and $\hat{K}_\ell \in \R^{D \times D}$ are no longer functions and operators, but random vectors and matrices. Similarly, we resample in each training iteration $t_\ell \sim \operatorname{Unif}(0,T)$ to avoid overfitting to any particular time discretization.

\textbf{Low-rank parameterization.}
While it is mathematically valid to parameterize $K_\theta(t)$ using a rank-$2$ generator in \eqref{eq:rank-2-def}, we generalize it to higher ranks $r \ge 2$ for more expressiveness. In particular,
\begin{equation}\label{eq:generator-parametrization}
    K_\theta(t) = U^\ast_\theta(t)\bigl(M_\theta(t) - M_\theta(t)^\top\bigr) U_\theta(t),
\end{equation}
where $U_\theta(t): H \to \mathbb{R}^r$ has $r$ rows (i.e.\ $r$ functions in $H$) and $M_\theta(t) \in \mathbb{R}^{r \times r}$ is an arbitrary matrix. 
By construction, $K_\theta(t)$ is skew-adjoint, has rank $\leq r$, and depends smoothly on $\theta$.
In practice, $U_\theta$ is implemented as a network that inputs both time $t$ and coordinate $\omega \in \Omega$ and outputs $r$ values, and $M_\theta$ is implemented as another network that inputs time $t$ and outputs an $r \times r$ matrix.
Numerically, $K_\theta$ is only evaluated on the midpoints $t_\ell + \tfrac{\Delta t_\ell}{2}$ and quadrature points $\hat{\Omega}$, producing $\hat{K}_\ell, \hat{U}_\ell$, and $\hat{U}_\ell^\ast$.

\textbf{Efficient Solver.} 
The quadrature set must be large enough ($1024 \le D \le 4096$ for most experiments) to accurately approximate the underlying function-space operations. Na\"ively applying $\hat{Q}_\ell$ to $\hat{\phi}_i^\theta$ would require $\cO(D^3)$ matrix vector products and linear solves; however, leveraging the low-rank structure of $\hat{K}_\ell$ allows us to be much more efficient: Each step amounts to applying $\hat{Q}_\ell$ to the input $\hat{\phi}_i^\theta$. Due to the low-rank structure, the forward action of applying $\bigl(I + \tfrac{1}{2}\hat{K}_\ell\bigr)$ can be done in $\mathcal{O}(r^2 + rD)$  time. On the other hand, for the inverse action $(I - \tfrac{1}{2}\hat{K}_\ell)^{-1}$, we apply the Woodbury identity \cite{higham2002accuracy}:
\begin{equation}\label{eq:woodbury}
    \left(I - \tfrac{1}{2}\hat{K}_\ell\right)^{-1} = I + \tfrac{1}{2} \hat{U}^\ast_\ell \hat{S}_\ell \left(I_{r \times r} - \tfrac{1}{2} \hat{U}_\ell \hat{U}^\ast_\ell \hat{S}_\ell\right)^{-1} \hat{U}_\ell,
\end{equation}
where $\hat{S}_\ell \coloneqq M_\theta(t_\ell + \tfrac{\Delta t_\ell}{2}) - M_\theta(t_\ell + \tfrac{\Delta t_\ell}{2})^\top$.
The Gram matrix $U_\theta U_\theta^\ast \in \mathbb{R}^{r \times r}$ is a small matrix of inner products, which we estimate on $\hat{\Omega}$ with $\cO(r^2D)$ operations. Hence, the total backward action reduces to $\cO(r^3+r^2 D)$ linear solves and matrix vector products, which are all linear in $D$.

\autoref{fig:low_rank_function_space_rotation} illustrates the algorithm in action. For full details on the integrator, the values of $L, r,$ and $D$, ablations, and the architecture used for each component, see \autoref{appx:lie-integration}. The algorithm presented here only involves real-valued functions, for experiments involving multiple channels, such as the CelebA dataset with RGB, we use a multichannel variant explained in the same appendix.

%% file: neurips_sections/experiments.tex
\begin{figure}[t]\captionsetup{font=footnotesize}
    \centering
    \includegraphics[width=\linewidth]{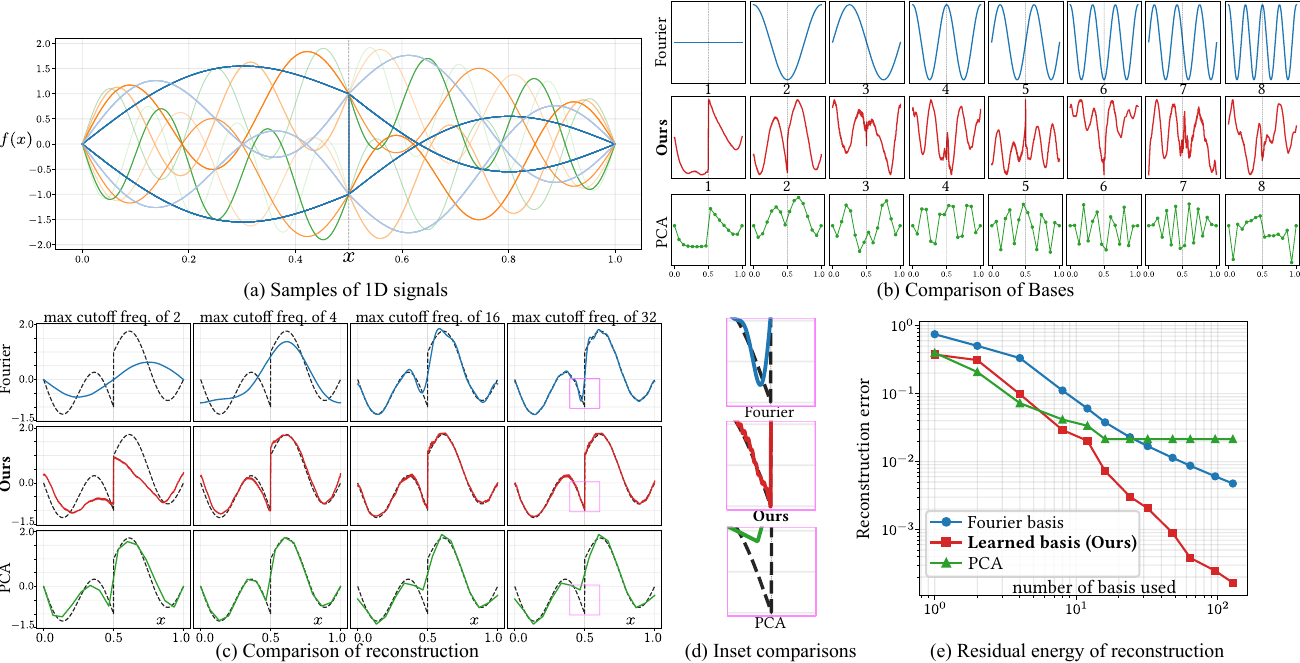}
    \caption{\textbf{1D Synthetic PCA on Function Space.} \textbf{(a)} Samples from a distribution of signals described by \eqref{eq:1d-synthetic}, exhibiting a jump at $x=0.5$ on a 1D domain. \textbf{(b)} The Fourier basis is transformed to diagonalize the covariance operator, allowing the basis to better capture the variance of signals. Our results improve upon the traditional finite-dimensional PCA construction, which is limited by the Nyquist cutoff inherent to that setting. \textbf{(c)} Reconstruction of several sample signals from the distribution. Our learned basis yields better reconstructions across all cutoff frequencies. \textbf{(d)} In particular, finite PCA is limited by Nyquist limits and the Fourier basis exhibits the Gibbs ringing phenomenon at the discontinuity. \textbf{(e)} A quantitative comparison of reconstructions averaged across the distribution of signals. The reconstruction error is shown on a log scale. We achieve an order-of-magnitude improvement using roughly 100 basis elements. }
    \label{fig:1D_fPCA}
    \vspace{-0.2cm}
\end{figure}

\section{Experiments} \label{sec:experiments}

In this section, we present experiments that evaluate our learned bases. Each task requires specific configuration choices. We focus on the primary takeaways and defer details and additional experiments to \autoref{appx:experiment-details}. Our codebase is available at \href{https://github.com/HamidrezaKmK/Learning-ONBs}{[https://github.com/HamidrezaKmK/Learning-ONBs]}.

\textbf{1D functional PCA.} We evaluate our approach using a synthetic 1D experiment on the domain $\Omega = [0, 1]$ equipped with a uniform measure $\mu$. To create a deliberately challenging scenario for standard Fourier methods, we sample functions with a sharp discontinuity at $x = 0.5$ (see \autoref{fig:1D_fPCA}):
\begin{equation} \label{eq:1d-synthetic}
    f(x) = \sigma \begin{cases}
        \sin(2\pi k x) + 2x & \text{if } x \le 0.5, \\
        -\sin(2\pi k x) + 2x - 2 & \text{if } x > 0.5,
    \end{cases}
\end{equation}
where the scale factor $\sigma \sim \text{Rademacher}$ takes values in $\{-1, 1\}$ with equal probability, and the frequency multiplier $k \sim \text{Geometric}(1/3)$ is a strictly positive integer. 
While discontinuous signals are notoriously difficult for standard Fourier reconstruction, our method overcomes this issue, also surpassing the baseline finite PCA reconstruction with as few as 10 basis functions. Moreover, as shown in \autoref{fig:1D_fPCA}, while finite PCA is optimal for discrete data, its grid-based discretization prevents it from fitting signals beyond the Nyquist rate. Our formulation bypasses this bottleneck.

\begin{figure}\captionsetup{font=footnotesize}
    \centering
    \includegraphics[width=\linewidth]{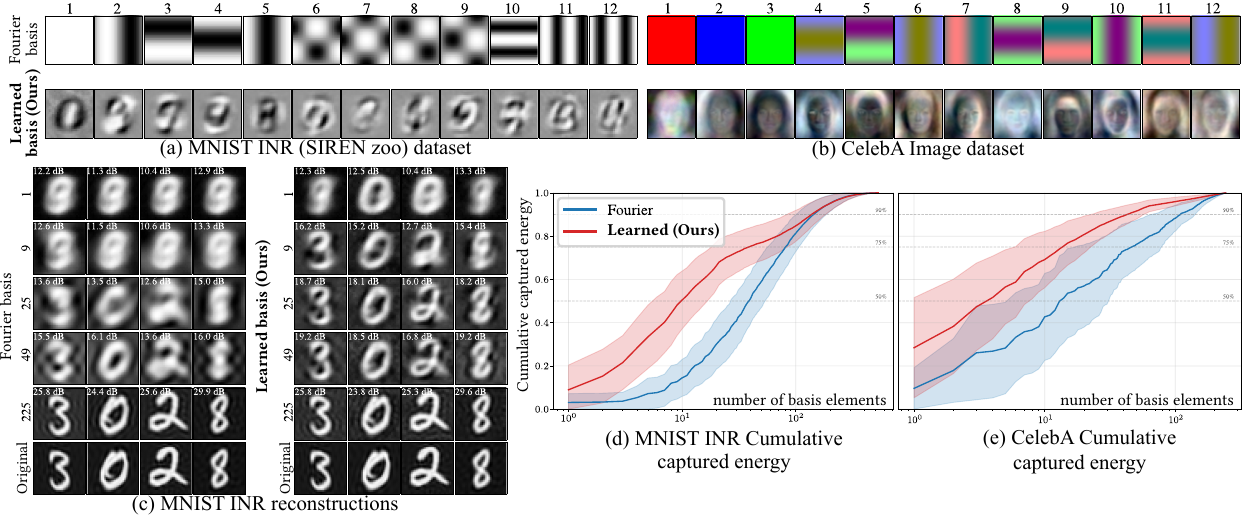}
    \caption{\textbf{2D PCA in Function Space.} \textbf{(a)} Fourier basis transformed to explain the INR zoo of SIREN networks representing the MNIST dataset. Note that the resulting basis represents a dataset that is itself discretization-free. \textbf{(b)} Transformed bases for the CelebA dataset. \textbf{(c)} Reconstruction of several datapoints using the Fourier basis and our transformed, optimized basis. With only a few basis elements, our basis reconstructs the signals more faithfully and captures the shape of the digits better than the Fourier baseline.  \textbf{(d)} Quantitative comparison of the cumulative captured energy (mean $\pm$ std across the dataset). Our learned basis captures more energy than the Fourier baseline. Because our basis is infinite, it gradually achieves full reconstruction without a Nyquist limit; any residual error stems from numerical integration.
    } \label{fig:fPCA_INR_zoo_and_CelebA}
    \vspace{-0.5cm}
\end{figure}

\textbf{Real Data PCA.} To maintain a discretization-free setting, we use the INR Zoo \cite{pmlr-v202-navon23a}, an MNIST \cite{lecun2010mnist} dataset where each datapoint is a SIREN network \cite{sitzmann2020implicit} overfitted to an image. 
We have a similar setup for a subset of 1024 images from CelebA, using 
quadrature points on a 
$64 \times 64$ grid since INRs were unavailable. 
In both cases, our method learns to transform the Fourier basis into principal functions of these datasets in \autoref{fig:fPCA_INR_zoo_and_CelebA}.
As illustrated, the learned basis achieves better reconstruction than the Fourier baseline and the first few basis elements capture most of the signal energy of the dataset.

\begin{figure}\captionsetup{font=footnotesize}
    \centering
    \includegraphics[width=\linewidth]{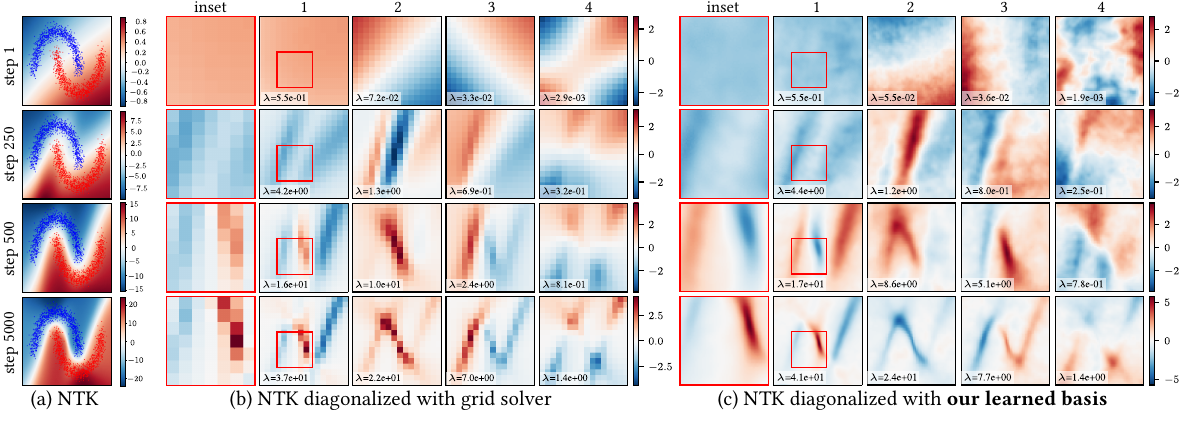}
    \caption{\textbf{Two-moon NTK.} Diagonalizing the NTK over the course of training a two-moon classifier (steps 1, 250, 500, and 5000). \textbf{(a)} The evolving logit landscape and training data. \textbf{(b)} Top eigenfunctions  recovered by a discrete grid-solver; the grid's fixed resolution leads to noticeable pixelation and a loss of high-frequency information. \textbf{(c)} Eigenfunctions obtained using our proposed discretization-free method. Our approach matches the global structure of the leading eigenfunctions while resolving the fine, intricate details near the decision boundary. Note that the eigenspaces are comparable up to sign which explains why the colormaps sometimes appear swapped between the two methods.}
    \label{fig:ntk_diagonalization}
\end{figure}

\textbf{Diagonalizing the neural tangent kernel (NTK).} Moving beyond PCA, we diagonalize the NTK operator at various stages of training a two-moon classifier. A short summary on what diagonalization of the NTK means is provided in \autoref{appx:diagonalization-theory}.
\autoref{fig:ntk_diagonalization} compares against a grid-based eigensolver. While the discrete solver isolates eigenfunctions with lower eigenvalues slightly better, its fixed resolution fails to capture fine spatial details, particularly at the sharp center of the decision boundary. Our discretization-free method successfully resolves these localized features, highlighting a fundamental advantage of solving eigenproblems entirely in function space.

\textbf{Koopman operators.} We train $Q_\theta$ to approximate the Koopman operator, by optimizing the objective in \eqref{eq:koopman-training}. Here, the underlying dynamical system $\Psi$ is the Taylor-Green vortex on a flat torus \cite{taylor1937mechanism}, simulated for a discrete time step of $\Delta t = 0.5s$. Specifically, the discrete flow map $\Psi$ is obtained by integrating the governing ODE $\dot{x} = v(x)$ over $\Delta t$ using a fourth-order Runge-Kutta (RK4) scheme, where $v$ parameterizes the Taylor-Green vortices. After training, we evaluate the model on a sample initial condition (see \autoref{fig:taylor_vortices}) by iteratively applying the learned operator to generate extended fluid-flow trajectory snapshots. We compare the energy conservation of these predicted rollouts against various solvers. Although baselines that run $\Psi$ may produce more visually plausible dynamics, they fail to preserve energy, evaluated as the Hilbert norm of the signal. In contrast, despite being noisier, our learned Koopman operator yields plausible results and strictly conserves energy.

\begin{figure}\captionsetup{font=footnotesize}
    \centering
    \includegraphics[width=\linewidth]{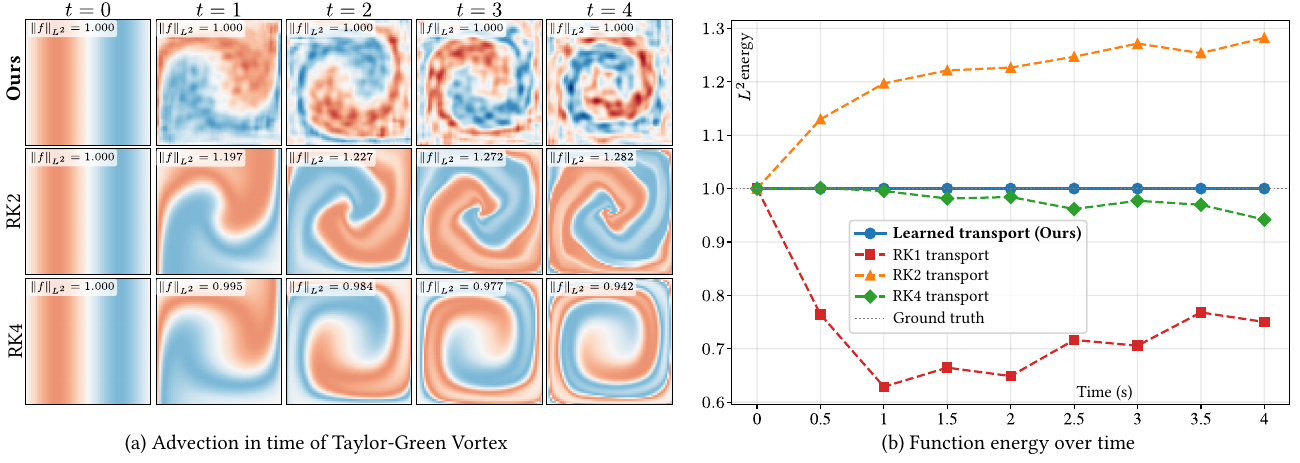}
    \caption{\textbf{Learning the Koopman Operator.} \textbf{(a)} Visual comparison of fluid-flow trajectories generated by our learned Koopman operator vs. the composition operator integrated via RK methods; our method is noisier but conserves energy. \textbf{(b)} $L^2$ energy tracking over time, comparing standard solvers (1\textsuperscript{st}, 2\textsuperscript{nd}, and 4\textsuperscript{th}-order RK) against iterative rollouts from $Q_\theta$.}
    \label{fig:taylor_vortices}
\end{figure}

%% file: neurips_sections/conclusions.tex
\section{Conclusions, Limitations, \& Future Work}

In this work, we introduced a neural framework for learning continuous orthonormal bases in function spaces. Our approach allows us to optimize these bases for various tasks while preserving their essential linear-algebraic properties, such as continuity and orthonormality. The Lie ODE formulation allowed us to prove theoretical guarantees on the expressiveness of these networks. However, optimizing and backpropagating through any ODE trajectory is well-known to be practically challenging \cite{gholami2019anode, finlay2020train}. Reaching our theoretical limits often demands long integration horizons, which can complicate optimization dynamics and increase training times. This motivates future work on tailored architectures and optimizers to scale our method.

Several exciting directions remain for future research. Here, we showed how one can diagonalize the covariance and NTK operators; a natural extension is to diagonalize elliptic operators in physics or the Laplace-Beltrami operator in differential geometry. Additionally, we showed how our change-of-basis operator can match the Koopman operator, which is also an integral component of the more general family of dynamic mode decomposition (DMD) \cite{schmid2010dynamic} methods. It is natural to extend our method in the context of DMD with more complicated physical systems and their Koopman operators.

In conclusion, we hope that parameterizing orthonormal bases as a continuous, optimizable neural network will offer a powerful new lens for modeling complex systems involving function spaces.

%% file: appendix/existence_proof.tex
\section{Proof of Theorem 1} \label{appx:theorem1-proof}

We first restate the theorem for completeness:

\ApproximationTheorem*

Here, we repeatedly use the fact that any rank-$2$ operator $K(t)$ can be written in the following form:
\begin{equation}
    K(t) = \alpha(t) \otimes \beta(t) - \beta(t) \otimes \alpha(t),
\end{equation}
where $\alpha, \beta: [0, T] \to H$ are $H$-valued functions of time.
Our goal is to ultimately show that for any operator $Q \in \mathcal{SO}(HS)$, there exist $\alpha, \beta$ such that the solution to the following initial value problem closely approximates $Q$ at $t=T$:
\begin{equation} \label{eq:cayley-generator}
    \dot{Q}(t) = \left[ \alpha(t) \otimes \beta(t) - \beta(t) \otimes \alpha(t) \right] Q(t), \qquad Q(0) = I.
\end{equation}
To do so, we first restrict our space to the following space of operators:
\begin{definition}
Let $\mathcal{SO}(\text{fin}) \subset \mathcal{SO}(HS)$ denote the group of ``finite-like'' orthogonal Hilbert Schmidt operators that act as identity outside of some finite-dimensional linear subspace of $V$ of $H$. That is, if $Q \in \mathcal{SO}(\text{fin})$ then there exists some finite $n$ and a finite set $S = \{e_1,\ldots,e_n\} \subset H$ such that
\begin{equation}
Q v = v \quad \text{for all } v \in V^\perp  = \text{span}\{e_1,\ldots,e_n\}^\perp.\nonumber
\end{equation}
\end{definition}

We now constructively prove that the approximation theorem holds for $\mathcal{SO}(\text{fin})$:

\begin{lemma}
For any $Q \in \mathcal{SO}(\text{fin})$, there exists continuous maps $\alpha, \beta : [0,T] \times \Omega \to \R$ such that the solution $Q(T)$ of \eqref{eq:cayley-generator} satisfies $Q(T) = Q$.
\end{lemma}

\begin{proof}
Let $Q \in \mathcal{SO}(\text{fin})$. Then, by definition, there exists a finite-dimensional subspace $V$ alongside its orthonormal basis $\{e_1, \ldots, e_n\} \subset H$ such that $Q$ acts trivially on $V^\perp$. 
We can therefore restrict the representation of $Q$ to $V$ and get a finite representation $Q_V \in SO(n)$. By classical results in numerical linear algebra and Lie group theory \cite{hall2013lie}, any element of $SO(n)$ can be written as a finite product of Givens' rotations acting on coordinate planes. 

In other words, for any $Q_V$ there exist $M$ angles $\{\theta_m\}_{m=1}^M$ and pairwise orthogonal directions $\{a_m\}_{m=1}^M, \{b_m\}_{m=1}^M \subset \mathbb{S}^{n-1}$ such that:
\begin{equation}
Q_V = \prod_{m=1}^M R^{\text{Givens}}_{a_m, b_m}(\theta_m) = \prod_{m=1}^M \exp\Big(\theta_m (a_m b_m^\top - b_m a_m^\top)\Big).
\end{equation}
We can now define $\alpha_m = \sqrt{\theta_m} \cdot \sum_{i=1}^n a_{mi} e_i \in H$ and $\beta_m = \sqrt{\theta_m} \cdot \sum_{i=1}^n b_{mi} e_i \in H$ which allows us to lift the representation back from the coefficient space to $H$ and write:
\begin{equation}
Q = \prod_{m=1}^M \exp\Big(\alpha_m \otimes \beta_m - \beta_m \otimes \alpha_m\Big).
\end{equation}
Now, the idea is to construct a continuous field in $H$, that once integrated, results in the operator above. We achieve this by stitching together $M$ time-evolving functions below:
\begin{align}
    \alpha_m(t) = r_m(t) \, \alpha_m, & \qquad \beta_m(t) = r_m(t)\, \beta_m.
\end{align}
We define $r_m: \R \to \R$ as a continuous control function in the form of a ``bump'' that has the following properties: $r_m(0) = r_m(1) = 0$ and $\int_0^1 r_m^2(\tau) d\tau = 1$. One such control function can be the square root of the probability density of the Beta distribution. 

These properties on $r_m(t)$ ensure $\alpha_m(1) = \alpha_{m+1}(0)$ and $\beta_m(1) = \beta_{m+1}(0)$, allowing us to stitch generators continuously and form a bigger generator. 
Furthermore, the construction above allows the self-adjoint operator to always self-commute over all $t \in [0, 1]$:
\begin{align}
    \alpha_m(t) \otimes \beta_m(t) - \beta_m(t) \otimes \alpha_m(t) = r_m^2(t) \left[ \alpha_m \otimes \beta_m - \beta_m \otimes \alpha_m \right]
\end{align}

Next, w.l.o.g. assume $T=M$ and define the following $M$ initial value problems:
\begin{equation}
    \dot{Q}_m = \left[ \alpha_m(t) \otimes \beta_m(t) - \beta_m(t) \otimes \alpha_m(t) \right] Q_m(t), \quad Q_m(0) = \begin{cases}
        Q_{m-1}(1) & \text{if } m > 1\\
        I & \text{otherwise}
    \end{cases}.
\end{equation}
Using the fact that the generator above is skew-adjoint and self-commutes, we can write the solution of the ODE as an exponential map and simplify:
\begin{align}
Q_m(1) &= \exp\left( \left[ \alpha_m \otimes \beta_m - \beta_m \otimes \alpha_m \right] \cdot \underset{(=1) \text{ by construction}}{\underbrace{ \int_0^1 r_m^2(\tau) d\tau}} \cdot\right)   Q_m(0)\\
\Rightarrow~~~ & Q_M(1) = \prod_{m=1}^M \exp\left( \alpha_m \otimes \beta_m - \beta_m \otimes \alpha_m \right) 
\end{align}
Hence, chaining them yields a time-evolving $Q(t)$ for the form \eqref{eq:cayley-generator} where $Q(T)=Q_M(1)=Q$.
\end{proof}

\begin{lemma}
$\mathcal{SO}(\text{fin})$ is dense in $\mathcal{SO}(HS)$ with respect to the Hilbert-Schmidt operator norm.
\end{lemma}

\begin{proof}
Let $Q \in \mathcal{SO}(HS)$, then from the definition, it can be written as the exponential map of a Hilbert-Schmidt skew-adjoint operator, i.e., $Q = \exp(K)$ for some $K = -K^\ast$ where $\|K\|_{HS} < \infty$.

Since $K$ is Hilbert-Schmidt, using the spectral theorem for operators, there exists a sequence of real values $\lambda_n > 0$ and an orthonormal set of vectors in $H$, $\{ u_n, v_n\}_{n=1}^\infty$, such that:
\begin{equation}
    K = \sum_{j=1}^\infty \lambda_j (v_j \otimes u_j - u_j \otimes v_j).
\end{equation}
Next, we can define the sequence of finite rank operators $\{K_n\}_{n=1}^\infty$ as the spectral truncation:
\begin{equation}
    K_n = \sum_{j=1}^n \lambda_j (v_j \otimes u_j - u_j \otimes v_j).
\end{equation}
Importantly, $\lim_{n \to \infty} K_n = K$ in $\| \cdot \|_{HS}$ and $K$ commutes with all $K_n$: $K K_n = K_n K$. 

Since $K_n$ is finite rank, its exponential acts as identity on $\text{range}(K_n)^{\perp}$; this means that $Q_n := \exp(K_n) \in \mathcal{SO}(\text{fin})$. Furthermore, we have that:
\begin{align}
    & \lim_{n \to \infty} \| Q_n - Q\|_{HS} \\
    =~ &  \lim_{n\to\infty} \| \exp(K_n) - \exp(K)\|_{HS} \\
    =~ & \lim_{n\to\infty} \| \exp(K) \left[ \exp(K_n - K) - I \right]\|_{HS} \qquad \text{\small Exponentiation of Commutative Operators}\\
    \le ~ &  \| \exp(K) \|_{HS} \cdot \| \exp\left(\lim_{n \to \infty}K_n - K \right)- I\|_{HS} \qquad \text{\small Continuity of norms and exponentials}\\
    = ~ & \| \exp(K) \|_{HS} \cdot \| \exp(\mathbf{0}) - I \|_{HS} = 0.
\end{align}
Thus, we have found for any $Q \in \mathcal{SO}(HS)$, a sequence of operators $Q_n \in \mathcal{SO}(\text{fin})$ such that $Q_n \overset{n\to \infty}\Longrightarrow Q$ in $\| \cdot \|_{HS}$; this is precisely the definition of $\mathcal{SO}(\text{fin})$ being dense in $\mathcal{SO}(HS)$.
\end{proof}

Collectively, the two lemmas above prove \autoref{thm:universal-appx}.

Finally, a universal approximation result with MLPs is straightforward: with a sufficiently overparameterized $\theta$, any two MLPs can approximate both functions $\alpha$ and $\beta$ on the compact domain $[0, T] \times \Omega$ \cite{hornik1989multilayer}. In turn, by approximating the continuous $\alpha: [0, T] \times \Omega \to \R$ and $\beta: [0, T] \times \Omega \to \R$, these MLPs implicitly approximate an arbitrary $Q \in \mathcal{SO}(HS)$.

%% file: appendix/diagonalization-theory.tex
\section{Diagonalizing Linear Operators} \label{appx:diagonalization-theory}

\subsection{Proof of Theorem 2}

\DiagonalizationTheorem*
\begin{proof}
Let $\Upsilon$ and $\Lambda$ be diagonal operators with entries $p(i)$ and $\lambda_i$, respectively, and let $\Phi \in \mathcal{SO}(\infty)$ be the orthogonal operator whose $i$-th column corresponds to $Q \varphi_i$. We can compactly rewrite the objective as $\operatorname{tr}(\Upsilon \Phi^\ast A \Phi)$. Let $\sigma_i(\cdot)$ an operator functional that inputs an operator and returns its $i$th largest eigenvalue.  Since both $\Upsilon$ and $A$ are trace-class, by applying von Neumann's trace inequality and the Poincar\'e separation theorem, respectively, we upper-bound the objective as follows: 
\begin{align}
    \operatorname{tr}(\Upsilon \Phi^\ast A \Phi) & \le \sum_{i=1}^\infty p(i) \sigma_i(\Phi^\ast A \Phi) \qquad \text{Von-Neumann inequality for trace-class operators}\\
    & \le \sum_{i=1}^\infty p(i) \sigma_i(A) \qquad \Phi \text{ is unitary, so: } \sigma_i(\Phi^\ast A \Phi) \le \sigma_i(A)\\
    & = \sum_{i = 1}^\infty p(i) \lambda_i.
\end{align}
The upper bound is exactly achieved when $\Phi \in \mathcal{SO}(HS)$ diagonalizes $A$.
\end{proof}

\subsection{Diagonalizing the Covariance Operator}

An important detail omitted from the main text is that, in practice, many datasets $X \sim \mathbb{P}$ used for principal component analysis are not necessarily zero-mean. To properly diagonalize the true covariance operator $\mathcal{C}$ using $Q_\theta$, we must shift the data by its empirical mean. 
Rather than computing the empirical mean beforehand, we accomplish the same effect by concurrently fitting a spatial mean function—parameterized as a neural field $\nu_{\psi}: \Omega \to \mathbb{R}$ using random Fourier features \cite{tancik2020fourier}. 

This leads to a joint objective where we maximize the captured energy of the centered data while minimizing the $L^2$ reconstruction error of the mean:
\begin{align}
\max_{\theta} ~~\;& \E_{i \sim p, X \sim \mathbb{P}} \langle X - \nu_\psi, Q_\theta \varphi_i\rangle^2, \\
\min_{\psi} ~~\;& \E_{X \sim \mathbb{P}}\|\nu_\psi - X\|^2.
\end{align}
Because the parameters $\psi$ and $\theta$ are decoupled across these two objectives, we can optimize them in parallel using independent optimizers for each task.

For clarity, the algorithm is summarized as a pseudo-code in \autoref{alg:pca}.

\begin{algorithm}[ht]
\caption{Continuous Principal Component Analysis}
\label{alg:pca}
\KwIn{Dataset $\mathcal{X}$, reference basis $\{\varphi_i\}_{i=1}$, distribution $p(i)$, domain $\Omega$ with measure $\mu$.}
\KwOut{Optimized basis $Q_\theta$ and mean function $\nu_\psi$.}
Initialize orthogonal network parameters $\theta$ and mean function parameters $\psi$\;
\While{not converged}{
    Sample a function $X \sim \mathcal{X}$, a basis index $i \sim p$, and quadrature points $\hat{\Omega} \sim \mu$\;
    
    $\mathcal{L}_{\mathrm{mean}} \leftarrow \frac{1}{|\hat{\Omega}|} \sum_{\omega \in \hat{\Omega}} \|\nu_\psi(\omega) - X(\omega)\|^2$\;
    $\mathcal{J}_{\mathrm{PCA}} \leftarrow \frac{1}{|\hat{\Omega}|} \sum_{\omega \in \hat{\Omega}} \langle X(\omega) - \nu_\psi(\omega), (Q_\theta \varphi_i)(\omega) \rangle^2$\;
    
    $\psi \leftarrow \text{update}(\psi,  \nabla_\xi \mathcal{L}_{\mathrm{mean}})$\;
    $\theta \leftarrow \text{update}(\theta, \nabla_\theta \mathcal{J}_{\mathrm{PCA}})$\;
}
\end{algorithm}

\subsection{Diagonalizing the Neural Tangent Kernel}\label{appx:ntk-theory}

Let $\text{MLP}(\cdot ;  \xi) : \Omega \to \R$ a non-linear neural function parameterized by $\xi$. We restrict our attention to networks with a scalar output; this encompasses the experimental setting in \autoref{sec:experiments} like our Two Moons experiment where the output of the network represents the logit of a binary classifier.

The Neural Tangent Kernel (NTK) is defined by the inner product of the network's parameter gradients. We can express this as an integral operator $\kappa_\xi$ acting on the function space $H$ with the kernel $\kappa_\xi[x, y] = \nabla_\xi \text{MLP}(x; \xi)^\top  \nabla_\xi \text{MLP}(y; \xi)$. Specifically, the action of the NTK operator on any function $f \in H$ is defined as:
\begin{equation}
    (\kappa_\xi f)(x) = \int_\Omega \nabla_\xi \text{MLP}(x; \xi)^\top  \nabla_\xi \text{MLP}(y; \xi) f(y) \, \mu (dy).
\end{equation}

To diagonalize the NTK using our framework, we substitute $\kappa_\xi$ for $A$ in the variational objective of \autoref{thm:diagonalization}. For a basis element $\phi_i = Q\varphi_i$, we can simplify the quadratic form $\langle \phi_i, \kappa_\xi \phi_i \rangle$:
\begin{align}
    & \langle \phi_i, \kappa_\xi \phi_i \rangle \\
    = ~~& \int_\Omega \phi_i(x) \left( \int_\Omega \nabla_\xi \text{MLP}(x; \xi)^\top  \nabla_\xi \text{MLP}(y; \xi) \phi_i(y) \, \mu (d y) \right) \mu(dx)\\
    = ~~& \left(\int_{\Omega} \left[\phi_i(x) \cdot  \nabla_\xi \text{MLP}(x; \xi)\right] \mu(dx)\right)^\top  \left(\int_{\Omega} \left[ \nabla_\xi \text{MLP}(x; \xi) \cdot \phi_i(x)\right] \mu(dx)\right) \quad \text{(Fubini)}\\
    = ~~& \| \E_{x \sim\mu} \left[ \phi_i(x) \nabla_\xi \text{MLP}(x; \xi)\right] \|_2^2.
\end{align}
Thus, the quadratic form is in fact equal to the squared norm of the expected parameter gradient. Substituting this simplified quadratic form back into the general objective, we arrive at our specialized objective function for diagonalizing the NTK:
\begin{equation}
    \max_{\theta} ~~\E_{i \sim p} \left\| \E_{x \sim \mu} \left[ (Q_\theta\varphi_i)(x) \nabla_\xi \text{MLP}(x; \xi) \right] \right\|^2_2.
\end{equation}

In practice, the optimization proceeds by sampling a target index $i \sim p$ and a batch of quadrature points $\hat{\Omega} \sim \mu$. At these points, we evaluate both our trial basis function $(Q_\theta \varphi_i)(x)$ and the per-sample parameter gradients $\nabla_\xi \text{MLP}(x; \xi)$. We then compute the empirical expectation over the batch by taking the weighted average of these evaluated gradients, yielding a single vector in the network's parameter space. Finally, we compute the squared norm of this expectation and backpropagate to update the basis parameters $\theta$, keeping the MLP parameters $\xi$ frozen. Upon convergence, the learned functions $\phi_i = Q_\theta \varphi_i$ faithfully represent the dominant eigenmodes of the training dynamics.

For clarity, the algorithm is summarized as a pseudo-code in \autoref{alg:ntk}.

\begin{algorithm}[H]
\caption{Diagonalizing the Neural Tangent Kernel}
\label{alg:ntk}
\KwIn{Domain $\Omega$ with measure $\mu$, reference basis $\{\varphi_i\}_{i=1}$, distribution $p(i)$, frozen $\text{MLP}(\cdot; \xi)$.}
\KwOut{Optimized NTK eigenfunctions $\phi_i = Q_\theta \varphi_i$.}
Initialize orthogonal network parameters $\theta$\;
\While{not converged}{
    Sample a batch of quadrature points $\hat{\Omega} \sim \mu$ \textbf{and} a basis index $i \sim p$\;
    
    \tcp{Step 1: Evaluate trial eigenfunctions and parameter gradients}
    \For{each $\omega \in \hat{\Omega}$}{
        $g(\omega) \leftarrow \nabla_\xi \mathrm{MLP}(\omega; \xi)$\;
        $q(\omega) \leftarrow (Q_\theta \varphi_i)(\omega)$\;
    }
    
    \tcp{Step 2: Compute empirical expected gradient and NTK objective}
    $v \leftarrow \frac{1}{|\hat{\Omega}|} \sum_{\omega \in \hat{\Omega}} q(\omega) g(\omega)$\;
    $\mathcal{J}_{\mathrm{NTK}} \leftarrow \|v\|_2^2$\;
    
    \tcp{Step 3: Update basis generator}
    $\theta \leftarrow \text{update}(\theta, \nabla_\theta \mathcal{J}_{\mathrm{NTK}})$\;
}
\end{algorithm}

%% file: appendix/lie-integration.tex
\section{Cayley Integration Details and Multichannel Details}
\label{appx:lie-integration}

This appendix expands on \autoref{sec:implementation_details}. We provide a pseudo-code for the Cayley integration in \autoref{alg:change-of-basis-cayley}; this method is internally implemented whenever $Q_\theta \varphi_i(\omega)$ is called for $\omega \in \hat{\Omega} \subset \Omega$. Notably, we set $T = 1$ for all experiments.

We also extend the Cayley integration to multichannel function spaces and  provide more details on the forward and backward Euler ODE solvers used in our calculations in \autoref{fig:low_rank_function_space_rotation}.

\begin{algorithm}[H]
\caption{Continuous Change-of-Basis via Cayley Integration}
\label{alg:change-of-basis-cayley}
\KwIn{Initial basis function $\varphi_i$, parameters $\theta$, rank $r$.}
\KwOut{Estimated transformed basis evaluation $\hat{\phi} \approx Q_\theta \varphi_i$.}
Sample random time discretization $0 = t_0 < t_1 < \dots < t_L = T$\;
Sample random dense spatial samples $\hat{\Omega} = \{\omega_j\}_{j=1}^D \sim \mu$\;
Initialize $\hat{\phi} \leftarrow \varphi_i(\hat{\Omega})$\;
\For{$\ell = 1, \dots, L$}{
    $\Delta t_\ell \leftarrow t_\ell - t_{\ell-1}$ \textbf{and} $t_{\mathrm{mid}} \leftarrow t_{\ell-1} + \Delta t_\ell / 2$\;

    $\hat{U} \leftarrow \sqrt{\Delta t_\ell}\, U_\theta(t_{\mathrm{mid}}, \hat{\Omega}) \in \mathbb{R}^{r \times D}$ \textbf{and} $\hat{S} \leftarrow M_\theta(t_{\mathrm{mid}}) - M_\theta(t_{\mathrm{mid}})^\top \in \mathbb{R}^{r \times r}$\;

    \tcp{Step 1: Estimate $r \times r$ Gram matrix and Woodbury inverse $M$}
    $G \leftarrow [\hat{U} \hat{U}^\top]_{i j} \approx \frac{1}{D} \sum_{\omega \in \hat{\Omega}} \hat{U}^i(\omega) \hat{U}^j(\omega) \in \mathbb{R}^{r \times r}$ \textbf{and} $M \leftarrow \left(I_{r \times r} - \tfrac{1}{2} G\, \hat{S}\right)^{-1}$\;

    \tcp{Step 2: Apply Forward Action $\hat{\phi} \mapsto \hat{\phi} + \tfrac{1}{2}\hat{U}^\top \hat{S} \hat{U} \hat{\phi}$}
    $c \leftarrow \hat{U} \hat{\phi} \approx \frac{1}{D} \sum_{\omega \in \hat{\Omega}} \hat{U}(\omega) \hat{\phi}(\omega) \in \mathbb{R}^{r}$\;
    $\hat{y} \leftarrow \hat{\phi} + \tfrac{1}{2}\, \hat{U}^\top\, \hat{S}\, c$\;

    \tcp{Step 3: Apply Backward Action via Woodbury}
    $\tilde{c} \leftarrow \hat{U} \hat{y} \approx \frac{1}{D} \sum_{\omega \in \hat{\Omega}} \hat{U}(\omega) \hat{y}(\omega) \in \mathbb{R}^{r}$\;
    $w \leftarrow \hat{S} M\, \tilde{c} \in \mathbb{R}^{r}$\;
    $\hat{\phi} \leftarrow \hat{y} + \tfrac{1}{2}\, \hat{U}^\top\, w$\;
}
\end{algorithm}

\subsection{Multichannel Hilbert Spaces}

The multichannel-valued Hilbert space $H_C=L^{2}(\Omega,\mathbb{R}^{C})$, is defined through the following inner product that integrates over the spatial domain and sums across channels:
\begin{equation}
\langle f,g\rangle = \sum_{c=1}^{C}\int_{\Omega}f_{c}(x)g_{c}(x)\,d\mu(x).
\end{equation}
Consequently, $H_C$ is a single unified Hilbert space rather than $C$ independent ones. Naturally, when $C=1$, $H_C$ simplifies to the standard real-valued Hilbert space denoted as $H$ in the main text.

For experiments involving multichannel data (e.g.\ RGB images in CelebA), we only need a minor adjustment to $U_\theta$. Instead of the single-channel formulation $U_\theta(t, \omega) \in \mathbb{R}^r$, we use $U_{\theta}(t,\omega) \in \mathbb{R}^{r C}$ and index by $U_\theta(t, \omega)^i_c$ to denote the $c$\textsuperscript{th} channel of the $i$\textsuperscript{th} row of $U_\theta$, or in short $\hat{U}^i_c$ if it is evaluated at some time and quadrature points $\hat{\Omega}$. Inner product approximations are updated to account for the channel dimension, meaning the Gram matrix and matrix-vector products become:
\begin{align}
    \left[ \hat{U} \hat{U}^\top\right]_{ij} &\approx \frac{1}{|\hat{\Omega}|} \sum_{\substack{\omega \in \hat{\Omega}\\ 1 \le c \le C}} \hat{U}(\omega)^i_c \cdot  \hat{U}(\omega)^j_c, \\
    \left[\hat{U} \hat{f}\right]_i &\approx \frac{1}{|\hat{\Omega}|} \sum_{\substack{\omega \in \hat{\Omega}\\ 1 \le c \le C}} \hat{U}(\omega)^i_c f_c(\omega).
\end{align}

If, instead, we applied generators independently per channel, we would have forced a block-diagonal structure that isolated the color channels. This would have prevented us from transferring energy between channels, instead, by operating on the unified multichannel Hilbert space $H_C$, our formulation allows mixing between channels. Empirically, this is what allows us to map a blue Fourier basis element to a multicolored function, as shown in \autoref{fig:low_rank_function_space_rotation}. Moreover, since $H_C$ itself is a proper Hilbert space, the approximation \autoref{thm:universal-appx} also holds in the multichannel case; our entire proof transfers by replacing the definition of the inner product.

\begin{figure}[t]\captionsetup{font=footnotesize}
    \centering
    \includegraphics[width=\linewidth]{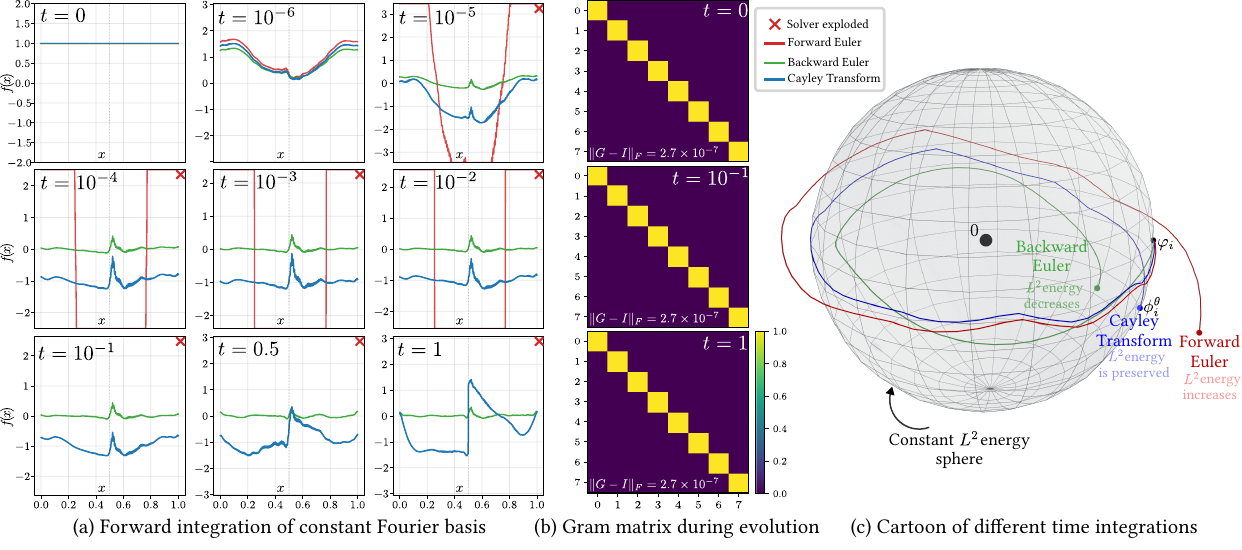}
    \caption{\textbf{Full Cayley Ablation.} Comparison of the Cayley method against backward and forward Euler. \textbf{(a)} Example evolution of a basis function $\phi_i^\theta(t)$ over time. The Cayley solver perfectly maintains signal energy, while forward Euler rapidly explodes and backward Euler dissipates. \textbf{(b)} Beyond norm preservation, the inner products of various basis functions, visualized as a gram matrix, also remain preserved through various snapshots of time. \textbf{(c)} A conceptual diagram of the Lie group integrator, which strictly remains on the orthogonal manifold, compared to methods that either overshoot or decay. Overall, Cayley integration is critical for stable optimization.}
    \label{fig:cayley_ablation}
\end{figure}

\subsection{Cayley vs.\ Off-the-Shelf ODE Solvers}
\label{appx:cayley-ablation}

\autoref{fig:cayley_ablation} compares our Cayley integrator with two classical alternatives, holding the number of steps $L$ and the skew-symmetric operator $\hat{K}_{\ell}$ constant:
\begin{itemize}[leftmargin=*]
\item \textbf{Forward (explicit) Euler}: The update formula can be written as $\hat{\phi}_\ell=(I+\Delta t\,\hat{K}_{\ell})\hat{\phi}_{\ell-1}$.
For a skew-symmetric $\hat{K}_\ell$, this is unstable as each step amplifies the norm by $\sqrt{1+(\Delta t\,\sigma_{K})^{2}} > 1$, where $\sigma_K$ represents the singular values of $\hat{K}_\ell$. This leads to the explosive growth shown in  \autoref{fig:low_rank_function_space_rotation} and \autoref{fig:cayley_ablation}.

\item \textbf{Backward (implicit) Euler}: 
The update rule can be written as $\hat{\phi}_{\ell}=(I-\Delta t\,\hat{K}_{\ell})^{-1}\hat{\phi}_{\ell-1}$.
This approach is stable but dissipative, with the norm decaying by $1/\sqrt{1+(\Delta t\,\sigma_{K})^{2}} < 1$ at each step. This leads to the dissipative performance shown in \autoref{fig:low_rank_function_space_rotation} and \autoref{fig:cayley_ablation}. For a fair comparison, we implement the inverse using the same Woodbury identity as our Cayley solver.

\item \textbf{Cayley (ours)}: Norm-preserving, ensuring $\|\hat{\phi}_\ell\|=1$ at every integration step.
\end{itemize}

%% file: appendix/experimental-details.tex
\section{Experimental Details}
\label{appx:experiment-details}

This appendix collects the configurations used for every experiment in
\autoref{sec:experiments}.  All runs share the same
Cayley integrator (\autoref{appx:lie-integration}) and differ only in their variational objective (PCA, NTK diagonalisation, or Koopman
fitting).
\autoref{tab:experiment-hparams} also summarized some key hyperparameters that are different across the different experiments.

\subsection{Indexing the Infinite Bases and their Prior $p(i)$}
\label{appx:index-prior}

All variational problems in \autoref{sec:variational} involve a real Fourier basis $\varphi_{i}$ defined on $\Omega = [0,1]^{d}$ and an index prior $p(i)$. While Fourier bases are enumerated by integers, for higher-dimensional domains, they are often enumerated with vector multi-indices. When $d > 1$, the index $i$ is represented as a vector of size $d$, $\mathbf{i} = [i_1, i_2, \ldots, i_d]$. For multichannel cases, the index is additionally extended to size $d+1$, $\mathbf{i} = [i_1, i_2, \ldots, i_d, c]$, for multi-channel tasks. For each spatial dimension $j$, the absolute value $|i_j|$ denotes the Fourier frequency. The sign of $i_j$ determines the trigonometric function: $i_j > 0$ yields a cosine, $i_j < 0$ yields a sine, and $i_j = 0$ yields the constant function. 

For example, in a 2D single-channel setting, the index $i = [2, -1]$ corresponds to the scalar basis function $\varphi_i(x_1, x_2) \propto \cos(2 \times 2\pi x_1) \sin(1 \times 2\pi x_2)$. In a multi-channel setting (e.g., a 3-channel RGB image), the extended index $i = [2, -1, 3]$ applies this exact same spatial feature strictly to the third channel. This yields the vector-valued basis function $\varphi_i(x_1, x_2) \propto \begin{bmatrix}
    0, & 0, & \cos(2 \times 2\pi x_1) \sin(1 \times 2\pi x_2)
\end{bmatrix}$.

We quantify the overall spatial frequency of a Fourier basis element using the squared $\ell^2$-norm of its spatial multi-index $\mathbf{i} = [i_1, \dots, i_d]$, given by $\|\mathbf{i}\|_2^2 = \sum_{j=1}^d i_j^2$. To prioritize lower-frequency features, we define the index prior $p$ as a strictly decreasing power law over this norm:
\begin{equation}\label{eq:isotropic-power-law}
p(i_1, \ldots, i_d, c) = \frac{1}{C\,Z_{\alpha}}\bigl(1+\|\mathbf{i}\|_{2}^{2}\bigr)^{-\alpha}, \qquad Z_{\alpha} \coloneqq \sum_{k \in \mathbb{Z}^{d}}\bigl(1+\|\mathbf{i}\|_{2}^{2}\bigr)^{-\alpha},
\end{equation}
where $C$ is the total channel count and $\alpha > d/2$ ensures the partition function $Z_{\alpha}$ is finite. Because \autoref{thm:diagonalization} requires a strict total ordering, any ties between distinct basis elements with identical frequency norms (e.g., $\mathbf{i}=(2,1)$ and $\mathbf{i}=(1,2)$) are resolved by applying a deterministic tiebreaker.

\paragraph{Objective Variance Reduction.}
Na\"ive Monte-Carlo estimation of variational objectives involving the expectation $\mathbb{E}_{\mathbf{i}\sim p}[\cdot]$ can suffer from high variance. To mitigate this, we use a \emph{stratified-with-tail} estimator that decomposes the expectation into two parts: an \emph{exact stratum} $S = \{\mathbf{i} : p(\mathbf{i}) \ge \tau\}$ comprising all high-probability indices, and a \emph{Monte-Carlo tail} consisting of $N_\tau$ independent and identically distributed (i.i.d.) draws from $p$, explicitly rejecting any indices already present in the exact stratum. Concretely, for an integrand $f$, the estimator is
\begin{equation}
    \hat{\mathbb{E}}[f] = \underbrace{\sum_{\mathbf{i} \in S} p(\mathbf{i})\, f(\mathbf{i})}_{\text{exact stratum}} \;+\; \underbrace{\frac{1}{N_\tau} \sum_{k=1}^{N_\tau} f(\mathbf{i}^{(k)})\, \mathbf{1}[\mathbf{i}^{(k)} \notin S]}_{\text{Monte-Carlo tail}}, \qquad \mathbf{i}^{(k)} \stackrel{\text{i.i.d.}}{\sim} p,
\end{equation}
where $\mathbf{1}(\cdot)$ is the indicator function of an event.
By construction, this estimator is unbiased: the exact stratum is evaluated with zero variance, while the variance of the sampled tail scales as $\mathcal{O}(1/N_\tau)$. The probability threshold $\tau$ and the tail sample size $N_\tau$ for each experiment are detailed in \autoref{tab:experiment-hparams}.

\begin{table}[h]\centering\footnotesize
\caption{Hyperparameters for every experiment shown in
\autoref{sec:experiments}.  Rank $r$, time-discretisation $L$, and
quadrature size $D$ are the parameters that control accuracy of the
function-space ODE; $\tau$ is the tail cutoff and $N_\tau$ the tail size for the stratified
energy estimator.}
\label{tab:experiment-hparams}
\begin{tabular}{lccccc}
\toprule
Experiment & $r$ & $L$ & $D$ & $\tau$ & $N_\tau$ \\
\midrule
1-D PCA (\autoref{fig:1D_fPCA})
& 30 & 20 & 64 & $10^{-3}$ & 16 \\
INR-Zoo MNIST (\autoref{fig:fPCA_INR_zoo_and_CelebA})
& 20 & 20 & $64^{2}$ & $10^{-4}$ & 64 \\
CelebA (\autoref{fig:fPCA_INR_zoo_and_CelebA})
& 50 & 20 & $64^{2}$ & $10^{-5}$ & 32 \\
NTK two-moons (\autoref{fig:ntk_diagonalization})
& 10 & 20 & $32^{2}$ & $7\!\times\!10^{-4}$ & 32 \\
Fluid Vortices (\autoref{fig:taylor_vortices})
& 10 & 50 & $32^{2}$ & $10^{-4}$ & 32 \\
\bottomrule
\end{tabular}
\end{table}

\subsection{Architecture}

All of our experiments parameterize the finite-rank basis generator using two neural networks: $M_\theta(t)$, which controls the internal mixing of the basis vectors, and $U_\theta(t, \omega)$, which projects the continuous spatial domain into the rank-$r$ subspace. Both networks are conditioned on time via a shared sinusoidal positional embedding $\gamma(t) \in \mathbb{R}^{512}$. This embedding consists of $F=256$ frequencies log-spaced in $[2^0, 2^3]$:
\begin{equation}
    \gamma(t) = \left[ \sin(2\pi\omega_1 t), \cos(2\pi\omega_1 t), \dots, \sin(2\pi\omega_F t), \cos(2\pi\omega_F t) \right], \quad \omega_f = 2^{3(f-1)/(F-1)}.
\end{equation}

\paragraph{Internal Mixing Matrix $M_\theta(t)$.} 
The network $M_\theta: \mathbb{R}^{512} \to \mathbb{R}^{r \times r}$ maps the time embedding $\gamma(t)$ to the unconstrained mixing matrix. We use a small two-hidden-layer MLP with hidden widths $[64, 64]$, SiLU activations \cite{elfwing2018sigmoid}, and RMSNorm \cite{zhang2019root} applied between the linear layers and activations. Because only the skew-symmetric component $M_\theta(t) - M_\theta(t)^\top$ enters the dynamics ODE, the raw output of this network does not require any structural constraints.

\paragraph{Spatial Projection Field $U_\theta(t, \omega)$.} 
The network mapping the spatio-temporal inputs $(\gamma(t), \omega) \mapsto U \in \mathbb{R}^{r}$ dictates how the spatial domain evolves. To efficiently capture high-frequency details while ensuring robust frequency mixing and representing the $1/k^2$ power spectrum typical of natural signals, we construct the final projection field $U_\theta(t, \omega)$ as the sum of a base NeRF-style multi-layer perceptron \cite{mildenhall2021nerf} $U^{\text{MLP}}_\theta(t, \omega)$ and a time-varying basis-element-selector residual:
\begin{equation}
    U_\theta(t, \omega) = U^{\text{MLP}}_\theta(t, \omega) + \sum_{|k|_\infty \le K-1} w_{k,\theta}(t) \frac{1}{1+\|k\|_2^2} \varphi_k(\omega).
\end{equation}
Across all experiments, the base network $U^{\text{MLP}}_\theta(t, \omega)$ concatenates $[\gamma(t), \omega, \eta(\omega)]$, where $\eta(\omega)$ represents multi-scale random Fourier features. We use $L=8$ levels log-spaced over frequencies $[1, 64]$, with $n_\ell = \max(1, 64/2^\ell)$ projections per level (yielding $\approx 254$ features). This input is fed into a 4-layer MLP of width $256$ with SiLU activations and RMSNorm, leading to a linear map to $\mathbb{R}^r$. For the residual component, the time-varying coefficients $w_{k,\theta}: \mathbb{R}^{512} \to \mathbb{R}^r$ are parameterized by a small 2-layer MLP operating on the time embedding $\gamma(t)$ (hidden widths $[64, 64]$, SiLU, RMSNorm). The $1 / (1+\|k\|_2^2)$ factor acts as a structural low-pass prior, while the bandwidth hyperparameter $K$ is set on a per-experiment basis.

\subsection{Compute and Optimization}
\label{appx:compute}
All experiments run on a single NVIDIA RTX 5090 (32 GB) using the Adam optimizer
\citep{kingma2014adam} with default hyperparameters and learning rate $10^{-3}$,
gradient clipping at $\|\nabla\|_{2}\le 1$, and gradient checkpointing through
the Cayley step ladder to keep VRAM constant in $L$. Each Cayley step costs
$\mathcal{O}(r^{3}+r^{2}D)$; for the typical $1024 \le D \le 4096$, $10 \le r \le 50$
this is well below the cost of a single MLP forward pass through $U_{\theta}$, so
wall-clock is dominated by the $L$ network evaluations of
$U_{\theta}(t_{\ell},\hat{\Omega})$ per training step. A typical 2-D run (CelebA RGB)
trains in $\sim 1.5$ hours.

\subsection{Additional Experiments}
\label{appx:additional}

\begin{figure}[h]\captionsetup{font=footnotesize}
    \centering
    \includegraphics[width=\linewidth]{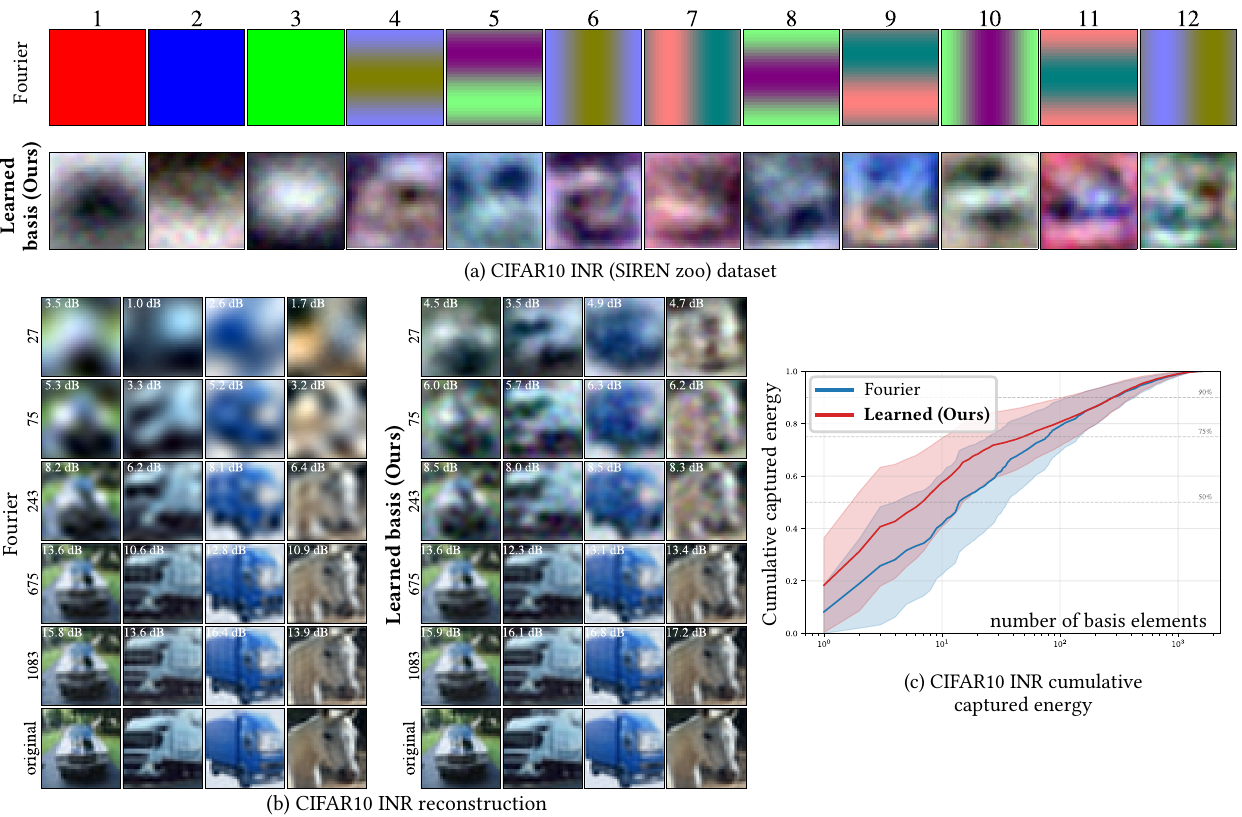}
    \caption{\textbf{PCA for CIFAR10 \cite{krizhevsky2009learning} SIREN Zoo.} \textbf{(a)} Visualizes the reference and learned basis after training; the learned bases resemble smooth principal components of the CIFAR10 dataset. \textbf{(b)} Reconstruction with the learned basis results in higher PSNRs. \textbf{(c)} Quantitatively, the captured energy is higher for our learned basis compared to the Fourier basis.}
    \label{fig:cifarten_inr_zoo}
\end{figure}

\begin{figure}[h]\captionsetup{font=footnotesize}
    \centering
    \includegraphics[width=\linewidth]{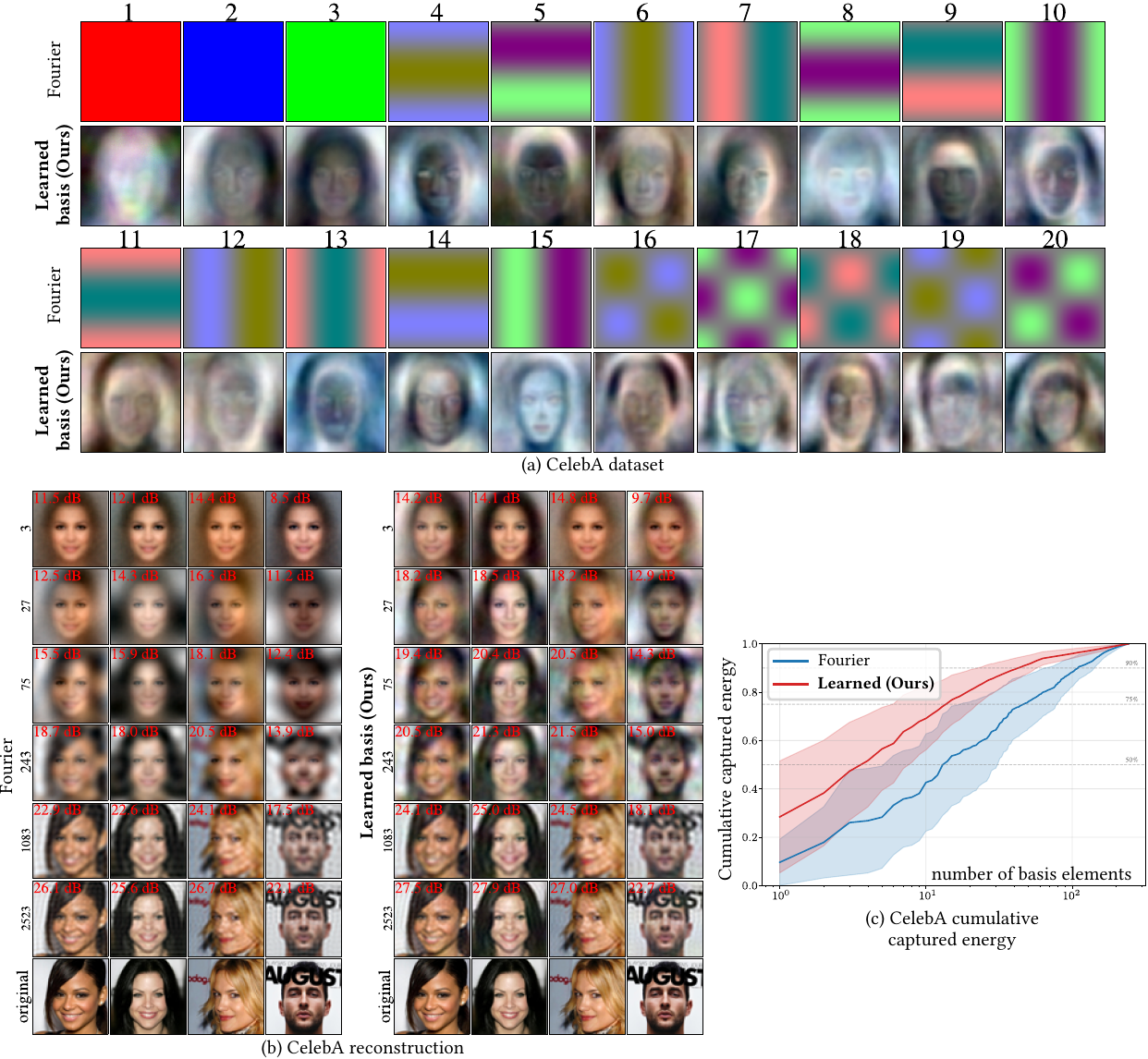}
    \caption{\textbf{PCA for CelebA.} \textbf{(a)} Visualizes the reference and learned basis; the learned bases resemble smooth ``eigenfaces'' of CelebA. \textbf{(b)} Reconstruction with the learned basis obtains higher PSNRs and captures the structure of the face better with fewer elements. \textbf{(c)} Quantitatively, the cumulative captured energy is higher for our learned basis compared to Fourier basis.}
    \label{fig:caleba_full}
\end{figure}

\begin{figure}[h]\captionsetup{font=footnotesize}
    \centering
    \includegraphics[width=\linewidth]{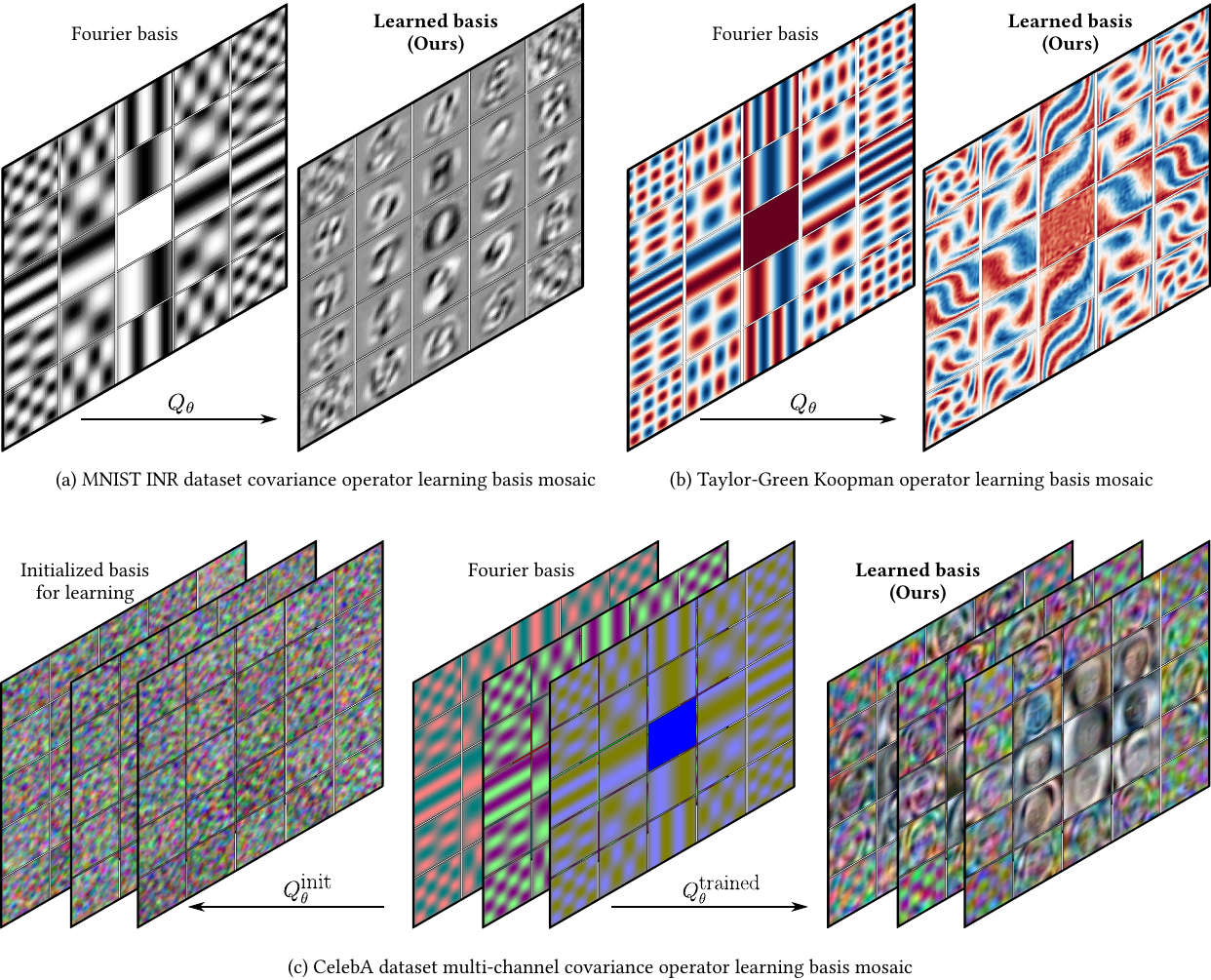}
    \caption{\textbf{Initial and Learned 2D and Multichannel Bases.} Visualization a truncated set of initial and learned bases across applications using our index notation (see \autoref{appx:index-prior}); the center of the mosaic corresponds to the direct current (DC) Fourier functions. \textbf{(a)} 2D Fourier and a learned bases fit to the principal components of the MNIST SIREN Zoo; the bases are visualized as a 2D mosaic indexed as $\varphi_{i_1,i_2}$ for Fourier and $Q_\theta \varphi_{i_1,i_2}$ for the learned basis where $i_1$ and $i_2$ correspond to the horizontal and vertical frequencies. \textbf{(b)} Similar results for the Vortices where each $Q_\theta \varphi_{i_1,i_2}$ corresponds to a dynamic mode of the Taylor-Green vortex field. \textbf{(c)} The 2D RGB-channel Fourier and learned basis for CelebA. Here the bases are visualized as a 2D mosaic with an extra channel dimension and indexed as $\varphi_{i_1,i_2,c}$. The basis $Q_\theta^{\text{init}} \varphi_{i_1,i_2,c}$ at initialization is obtained from a random path on the Lie manifold, but once trained, it learns to map to the eigenmodes of the covariance operator.}
    \label{fig:basis_mosaic}
\end{figure}

\begin{figure}
    \centering
    \includegraphics[width=\linewidth]{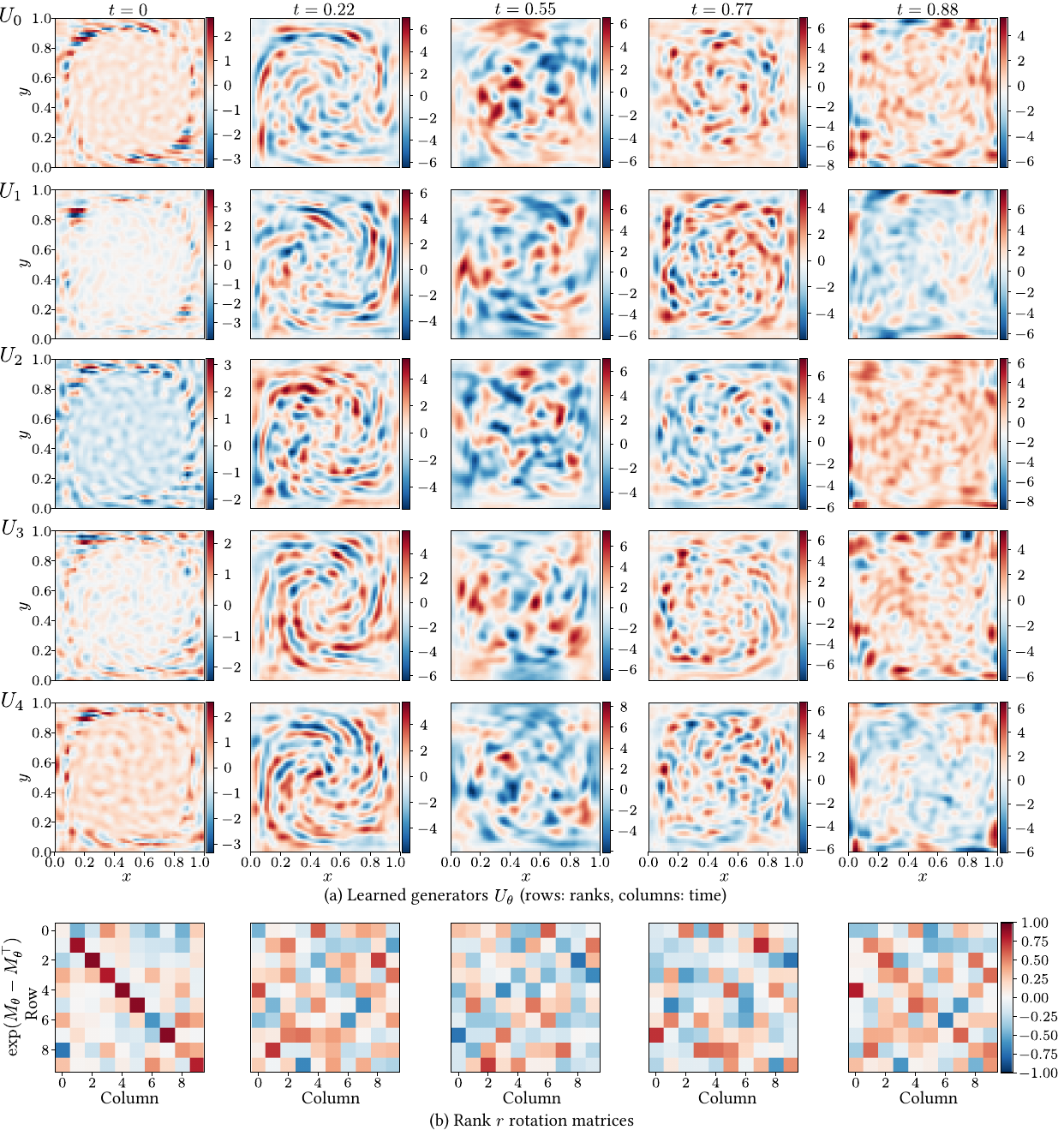}
    \caption{\textbf{Internal Network Dynamics.} This figure illustrates the learned neural network outputs that transform the Fourier basis into the dynamic modes of the Taylor-Green vortices. Specifically, we visualize the first five neural fields, $U_\theta^{0 \dots 4}$, across various time steps \textbf{(a)}, alongside the spatial rotations $\exp(M_\theta - M_\theta^\top)$ \textbf{(b)} that jointly characterize the generator $K_\theta(t)$. Passing the Fourier basis through these rank-$r$ planar rotations allows us to simulate the Koopman operator.}
    \label{fig:generators}
\end{figure}

Additional experiments and visualizations are provided in 
\autoref{fig:cifarten_inr_zoo}--\autoref{fig:generators}. 
\autoref{fig:cifarten_inr_zoo} evaluates the learned basis on the CIFAR10 SIREN 
Zoo \citep{krizhevsky2009learning}, showing that the learned bases resemble smooth 
principal components of the dataset and yield higher PSNRs and greater captured 
energy than the Fourier basis. \autoref{fig:caleba_full} presents analogous results 
for CelebA, where the learned bases take the form of smooth eigenfaces and achieve 
better reconstruction quality and higher cumulative captured energy with fewer basis 
elements. \autoref{fig:basis_mosaic} visualizes a truncated set of initial and 
learned bases across applications---including MNIST, Taylor-Green vortices, and 
CelebA---demonstrating that while the initialization corresponds to a random path on 
the Lie manifold, training recovers the eigenmodes of the covariance operator. 
Finally, \autoref{fig:generators} illustrates the internal network dynamics for the 
vortex experiment, visualizing the learned neural fields $U_\theta^{0\dots4}$ and 
spatial rotations $\exp(M_\theta - M_\theta^\top)$ that define the generator 
$K_\theta(t)$ and enable simulation of the Koopman operator via rank-$r$ planar 
rotations of the Fourier basis.

%% file: main.bib
@article{dupont2022data,
  title={From data to functa: Your data point is a function and you can treat it like one},
  author={Dupont, Emilien and Kim, Hyunjik and Eslami, SM and Rezende, Danilo and Rosenbaum, Dan},
  journal={arXiv preprint arXiv:2201.12204},
  year={2022}
}

@article{elfwing2018sigmoid,
  title={Sigmoid-weighted linear units for neural network function approximation in reinforcement learning},
  author={Elfwing, Stefan and Uchibe, Eiji and Doya, Kenji},
  journal={Neural networks},
  volume={107},
  pages={3--11},
  year={2018},
  publisher={Elsevier}
}

@article{mildenhall2021nerf,
  title={Nerf: Representing scenes as neural radiance fields for view synthesis},
  author={Mildenhall, Ben and Srinivasan, Pratul P and Tancik, Matthew and Barron, Jonathan T and Ramamoorthi, Ravi and Ng, Ren},
  journal={Communications of the ACM},
  volume={65},
  number={1},
  pages={99--106},
  year={2021},
  publisher={ACM New York, NY, USA}
}

@article{zhang2019root,
  title={Root mean square layer normalization},
  author={Zhang, Biao and Sennrich, Rico},
  journal={Advances in neural information processing systems},
  volume={32},
  year={2019}
}

@article{koopman1931hamiltonian,
  title={Hamiltonian systems and transformation in Hilbert space},
  author={Koopman, Bernard O},
  journal={Proceedings of the National Academy of Sciences},
  volume={17},
  number={5},
  pages={315--318},
  year={1931}
}

@article{lecun2010mnist,
  title={Gradient-based learning applied to document recognition},
  author={LeCun, Yann and Bottou, L{\'e}on and Bengio, Yoshua and Haffner, Patrick},
  journal={Proceedings of the IEEE},
  volume={86},
  number={11},
  pages={2278--2324},
  year={1998},
  publisher={IEEE}
}

@article{chen2018neural,
  title={Neural ordinary differential equations},
  author={Chen, Ricky TQ and Rubanova, Yulia and Bettencourt, Jesse and Duvenaud, David K},
  journal={Advances in neural information processing systems},
  volume={31},
  year={2018}
}

@book{golub2013matrix,
  title={Matrix computations},
  author={Golub, Gene H and Van Loan, Charles F},
  year={2013},
  publisher={JHU press}
}

@article{taylor1937mechanism,
  title={Mechanism of the production of small eddies from large ones},
  author={Taylor, Geoffrey Ingram and Green, Albert Edward},
  journal={Proceedings of the Royal Society of London. Series A-Mathematical and Physical Sciences},
  volume={158},
  number={895},
  pages={499--521},
  year={1937},
  publisher={The Royal Society London}
}

@article{gallier2006remarks,
  title={Remarks on the Cayley Representation of Orthogonal Matrices and on Perturbing the Diagonal of a Matrix to Make it Invertible},
  author={Gallier, Jean},
  journal={arXiv preprint math/0606320},
  year={2006}
}

@misc{kist_andreas_m_2021_5561092,
  author       = {Kist, Andreas M},
  title        = {{CelebA Dataset cropped with Haar-Cascade face detector}},
  year         = 2021,
  publisher    = {Zenodo},
  doi          = {10.5281/zenodo.5561092},
  url          = {https://doi.org/10.5281/zenodo.5561092}
}

@InProceedings{pmlr-v202-navon23a,
  title = 	 {Equivariant Architectures for Learning in Deep Weight Spaces},
  author =       {Navon, Aviv and Shamsian, Aviv and Achituve, Idan and Fetaya, Ethan and Chechik, Gal and Maron, Haggai},
  booktitle = 	 {Proceedings of the 40th International Conference on Machine Learning},
  pages = 	 {25790--25816},
  year = 	 {2023},
  editor = 	 {Krause, Andreas and Brunskill, Emma and Cho, Kyunghyun and Engelhardt, Barbara and Sabato, Sivan and Scarlett, Jonathan},
  volume = 	 {202},
  series = 	 {Proceedings of Machine Learning Research},
  month = 	 {23--29 Jul},
  publisher =    {PMLR},
  url = 	 {https://proceedings.mlr.press/v202/navon23a.html},
}

@article{kovachki2023neural,
  title={Neural operator: Learning maps between function spaces with applications to pdes},
  author={Kovachki, Nikola and Li, Zongyi and Liu, Burigede and Azizzadenesheli, Kamyar and Bhattacharya, Kaushik and Stuart, Andrew and Anandkumar, Anima},
  journal={Journal of Machine Learning Research},
  volume={24},
  number={89},
  pages={1--97},
  year={2023}
}

@article{hornik1989multilayer,
  title={Multilayer feedforward networks are universal approximators},
  author={Hornik, Kurt and Stinchcombe, Maxwell and White, Halbert},
  journal={Neural networks},
  volume={2},
  number={5},
  pages={359--366},
  year={1989},
  publisher={Elsevier}
}

@book{mallat1999wavelet,
  title={A wavelet tour of signal processing},
  author={Mallat, St{\'e}phane},
  year={1999},
  publisher={Elsevier}
}

@inproceedings{gholami2019anode,
  title     = {ANODE:  Unconditionally Accurate Memory-Efficient Gradients for Neural ODEs},
  author    = {Gholaminejad, Amir and Keutzer, Kurt and Biros, George},
  booktitle = {Proceedings of the Twenty-Eighth International Joint Conference on
               Artificial Intelligence, {IJCAI-19}},
  publisher = {International Joint Conferences on Artificial Intelligence Organization},
  pages     = {730--736},
  year      = {2019},
  month     = {7},
  doi       = {10.24963/ijcai.2019/103},
  url       = {https://doi.org/10.24963/ijcai.2019/103},
}

@article{krizhevsky2009learning,
  title={Learning multiple layers of features from tiny images},
  author={Krizhevsky, Alex and Hinton, Geoffrey},
  year={2009},
  journal={Technical Report},
  publisher={University of Toronto}
}

@inproceedings{finlay2020train,
  title={How to train your neural ode: the world of jacobian and kinetic regularization},
  author={Finlay, Chris and Jacobsen, J{\"o}rn-Henrik and Nurbekyan, Levon and Oberman, Adam},
  booktitle={International conference on machine learning},
  pages={3154--3164},
  year={2020},
  organization={PMLR}
}

@article{de2023deep,
  title={Deep learning on implicit neural representations of shapes},
  author={De Luigi, Luca and Cardace, Adriano and Spezialetti, Riccardo and Ramirez, Pierluigi Zama and Salti, Samuele and Di Stefano, Luigi},
  journal={The Eleventh International Conference on Learning Representations},
  year={2023}
}

@article{xu2022signal,
  title={Signal processing for implicit neural representations},
  author={Xu, Dejia and Wang, Peihao and Jiang, Yifan and Fan, Zhiwen and Wang, Zhangyang},
  journal={Advances in Neural Information Processing Systems},
  volume={35},
  pages={13404--13418},
  year={2022}
}

@book{stark1986probability,
  title={Probability, random processes, and estimation theory for engineers},
  author={Stark, Henry and Woods, John W},
  year={1986},
  publisher={Prentice-Hall, Inc.}
}

@book{ghanem2003stochastic,
  title={Stochastic finite elements: a spectral approach},
  author={Ghanem, Roger G and Spanos, Pol D},
  year={2003},
  publisher={Courier Corporation}
}

@article{wang2016functional,
  title={Functional data analysis},
  author={Wang, Jane-Ling and Chiou, Jeng-Min and M{\"u}ller, Hans-Georg},
  journal={Annual Review of Statistics and its application},
  volume={3},
  pages={257--295},
  year={2016},
  publisher={Annual Reviews}
}

@article{sitzmann2020implicit,
  title={Implicit neural representations with periodic activation functions},
  author={Sitzmann, Vincent and Martel, Julien and Bergman, Alexander and Lindell, David and Wetzstein, Gordon},
  journal={Advances in neural information processing systems},
  volume={33},
  pages={7462--7473},
  year={2020}
}

@inproceedings{xie2022neural,
  title={Neural fields in visual computing and beyond},
  author={Xie, Yiheng and Takikawa, Towaki and Saito, Shunsuke and Litany, Or and Yan, Shiqin and Khan, Numair and Tombari, Federico and Tompkin, James and Sitzmann, Vincent and Sridhar, Srinath},
  booktitle={Computer graphics forum},
  volume={41},
  number={2},
  pages={641--676},
  year={2022},
  organization={Wiley Online Library}
}

@incollection{hall2013lie,
  title={Lie groups, Lie algebras, and representations},
  author={Hall, Brian C},
  booktitle={Quantum Theory for Mathematicians},
  pages={333--366},
  year={2013},
  publisher={Springer}
}

@article{li2024physics,
  title={Physics-informed neural operator for learning partial differential equations},
  author={Li, Zongyi and Zheng, Hongkai and Kovachki, Nikola and Jin, David and Chen, Haoxuan and Liu, Burigede and Azizzadenesheli, Kamyar and Anandkumar, Anima},
  journal={ACM/JMS Journal of Data Science},
  volume={1},
  number={3},
  pages={1--27},
  year={2024},
  publisher={ACM New York, NY}
}

@article{li2020fourier,
  title={Fourier neural operator for parametric partial differential equations},
  author={Li, Zongyi and Kovachki, Nikola and Azizzadenesheli, Kamyar and Liu, Burigede and Bhattacharya, Kaushik and Stuart, Andrew and Anandkumar, Anima},
  journal={International Conference on Learning Representations},
  year={2021}
}

@article{li2020neural,
  title={Neural operator: Graph kernel network for partial differential equations},
  author={Li, Zongyi and Kovachki, Nikola and Azizzadenesheli, Kamyar and Liu, Burigede and Bhattacharya, Kaushik and Stuart, Andrew and Anandkumar, Anima},
  journal={arXiv preprint arXiv:2003.03485},
  year={2020}
}

@book{gottlieb1977numerical,
  title={Numerical analysis of spectral methods: theory and applications},
  author={Gottlieb, David and Orszag, Steven A},
  year={1977},
  publisher={SIAM}
}

@book{trefethen2000spectral,
  title={Spectral methods in MATLAB},
  author={Trefethen, Lloyd N},
  year={2000},
  publisher={SIAM}
}

@book{bremaud2002mathematical,
  title={Mathematical principles of signal processing: Fourier and wavelet analysis},
  author={Br{\'e}maud, Pierre},
  year={2002},
  publisher={Springer}
}

@inproceedings{levy2006laplace,
  title={Laplace-beltrami eigenfunctions towards an algorithm that" understands" geometry},
  author={L{\'e}vy, Bruno},
  booktitle={IEEE International Conference on Shape Modeling and Applications 2006 (SMI'06)},
  pages={13--13},
  year={2006},
  organization={IEEE}
}

@article{chang2024shape,
    author = {Chang, Yue and Benchekroun, Otman and Chiaramonte, Maurizio M. and Chen, Peter Yichen and Grinspun, Eitan},
    title = {Shape Space Spectra},
    year = {2025},
    issue_date = {August 2025},
    publisher = {Association for Computing Machinery},
    address = {New York, NY, USA},
    volume = {44},
    number = {4},
    issn = {0730-0301},
    url = {https://doi.org/10.1145/3731148},
    doi = {10.1145/3731148},
    journal = {ACM Trans. Graph.},
    month = jul,
    articleno = {121},
    numpages = {16}
}

@article{abdi2010principal,
  title={Principal component analysis},
  author={Abdi, Herv{\'e} and Williams, Lynne J},
  journal={Wiley interdisciplinary reviews: computational statistics},
  volume={2},
  number={4},
  pages={433--459},
  year={2010},
  publisher={Wiley Online Library}
}

@article{gerbrands1981relationships,
  title={On the relationships between SVD, KLT and PCA},
  author={Gerbrands, Jan J},
  journal={Pattern recognition},
  volume={14},
  number={1-6},
  pages={375--381},
  year={1981},
  publisher={Elsevier}
}

@article{tancik2020fourier,
  title={Fourier features let networks learn high frequency functions in low dimensional domains},
  author={Tancik, Matthew and Srinivasan, Pratul and Mildenhall, Ben and Fridovich-Keil, Sara and Raghavan, Nithin and Singhal, Utkarsh and Ramamoorthi, Ravi and Barron, Jonathan and Ng, Ren},
  journal={Advances in neural information processing systems},
  volume={33},
  pages={7537--7547},
  year={2020}
}

@article{modi2024simplicits,
  title={Simplicits: Mesh-free, geometry-agnostic elastic simulation},
  author={Modi, Vismay and Sharp, Nicholas and Perel, Or and Sueda, Shinjiro and Levin, David IW},
  journal={ACM Transactions on Graphics (TOG)},
  volume={43},
  number={4},
  pages={1--11},
  year={2024},
  publisher={ACM New York, NY, USA}
}

@inproceedings{sharp2023data,
  title={Data-free learning of reduced-order kinematics},
  author={Sharp, Nicholas and Romero, Cristian and Jacobson, Alec and Vouga, Etienne and Kry, Paul and Levin, David IW and Solomon, Justin},
  booktitle={ACM SIGGRAPH 2023 Conference Proceedings},
  pages={1--9},
  year={2023}
}

@inproceedings{azmoodeh2025continuous,
  title={Continuous-Time Signal Decomposition: An Implicit Neural Generalization of PCA and ICA},
  author={Azmoodeh, Shayan K and Subramani, Krishna and Smaragdis, Paris},
  booktitle={2025 IEEE 35th International Workshop on Machine Learning for Signal Processing (MLSP)},
  pages={1--6},
  year={2025},
  organization={IEEE}
}

@inproceedings{deng2022neuralef,
  title={Neuralef: Deconstructing kernels by deep neural networks},
  author={Deng, Zhijie and Shi, Jiaxin and Zhu, Jun},
  booktitle={International Conference on Machine Learning},
  pages={4976--4992},
  year={2022},
  organization={PMLR}
}

@article{ben2023deep,
  title={Deep learning solution of the eigenvalue problem for differential operators},
  author={Ben-Shaul, Ido and Bar, Leah and Fishelov, Dalia and Sochen, Nir},
  journal={Neural Computation},
  volume={35},
  number={6},
  pages={1100--1134},
  year={2023},
  publisher={MIT Press One Rogers Street, Cambridge, MA 02142-1209, USA journals-info~…}
}

@article{pfau2018spectral,
  title={Spectral inference networks: Unifying deep and spectral learning},
  author={Pfau, David and Petersen, Stig and Agarwal, Ashish and Barrett, David GT and Stachenfeld, Kimberly L},
  journal={International Conference on Learning Representations},
  year={2019}
}

@article{gemp2020eigengame,
  title={Eigengame: Pca as a nash equilibrium},
  author={Gemp, Ian and McWilliams, Brian and Vernade, Claire and Graepel, Thore},
  journal={International Conference on Learning Representations},
  year={2021}
}

@book{saad2011numerical,
  title={Numerical methods for large eigenvalue problems: revised edition},
  author={Saad, Yousef},
  year={2011},
  publisher={SIAM}
}

@inproceedings{yuce2022structured,
  title={A structured dictionary perspective on implicit neural representations},
  author={Y{\"u}ce, Gizem and Ortiz-Jim{\'e}nez, Guillermo and Besbinar, Beril and Frossard, Pascal},
  booktitle={Proceedings of the IEEE/CVF Conference on Computer Vision and Pattern Recognition},
  pages={19228--19238},
  year={2022}
}

@article{bengio2003out,
  title={Out-of-sample extensions for lle, isomap, mds, eigenmaps, and spectral clustering},
  author={Bengio, Yoshua and Paiement, Jean-fran{\c{c}}cois and Vincent, Pascal and Delalleau, Olivier and Roux, Nicolas and Ouimet, Marie},
  journal={Advances in neural information processing systems},
  volume={16},
  year={2003}
}

@article{ashtari2026futon,
  title={FUTON: Fourier Tensor Network for Implicit Neural Representations},
  author={Ashtari, Pooya and Behmandpoor, Pourya and Deligiannis, Nikos and Pizurica, Aleksandra},
  journal={arXiv preprint arXiv: 2602.13414},
  arxiv = {https://arxiv.org/abs/2602.13414},
  year={2026}
}

@article{kingma2014adam,
  title={Adam: A method for stochastic optimization},
  author={Kingma, Diederik P and Ba, Jimmy},
  journal={arXiv preprint arXiv:1412.6980},
  year={2014}
}

@article{lu2021learning,
  title={Learning nonlinear operators via DeepONet based on the universal approximation theorem of operators},
  author={Lu, Lu and Jin, Pengzhan and Pang, Guofeng and Zhang, Zhongqiang and Karniadakis, George Em},
  journal={Nature machine intelligence},
  volume={3},
  number={3},
  pages={218--229},
  year={2021}
}

@article{raissi2019physics,
  title={Physics-informed neural networks: A deep learning framework for solving forward and inverse problems involving nonlinear partial differential equations},
  author={Raissi, Maziar and Perdikaris, Paris and Karniadakis, George E},
  journal={Journal of Computational physics},
  volume={378},
  pages={686--707},
  year={2019}
}

@article{courant1954methods,
  title={Methods of mathematical physics, vol. I},
  author={Courant, Richard and Hilbert, David},
  journal={Phys. Today},
  volume={7},
  number={5},
  pages={17--17},
  year={1954}
}

@article{schmid2010dynamic,
  title={Dynamic mode decomposition of numerical and experimental data},
  author={Schmid, Peter J},
  journal={Journal of fluid mechanics},
  volume={656},
  pages={5--28},
  year={2010},
  publisher={Cambridge University Press}
}

@book{stoica2005spectral,
  title={Spectral analysis of signals},
  author={Stoica, Petre and Moses, Randolph L and others},
  volume={452},
  year={2005},
  publisher={Pearson Prentice Hall Upper Saddle River, NJ}
}

@article{iserles2001cayley,
  title={On Cayley-transform methods for the discretization of Lie-group equations},
  author={Iserles, Arieh},
  journal={Foundations of Computational Mathematics},
  volume={1},
  number={2},
  pages={129--160},
  year={2001},
  publisher={Springer}
}

@article{diele1998cayley,
  title={The Cayley transform in the numerical solution of unitary differential systems},
  author={Diele, Fasma and Lopez, Luciano and Peluso, R},
  journal={Advances in computational mathematics},
  volume={8},
  number={4},
  pages={317--334},
  year={1998},
  publisher={Springer}
}

@article{celledoni2014introduction,
  title={An introduction to Lie group integrators--basics, new developments and applications},
  author={Celledoni, Elena and Marthinsen, H{\aa}kon and Owren, Brynjulf},
  journal={Journal of Computational Physics},
  volume={257},
  pages={1040--1061},
  year={2014},
  publisher={Elsevier}
}

@book{higham2002accuracy,
  title={Accuracy and stability of numerical algorithms},
  author={Higham, Nicholas J},
  year={2002},
  publisher={SIAM}
}
